\pdfoutput=1
\documentclass{article}
\usepackage[numbers]{natbib}
\usepackage[final]{neurips_2021}
\usepackage[utf8]{inputenc} 
\usepackage[T1]{fontenc}    
\usepackage{xcolor}

\usepackage{hyperref}
\hypersetup{colorlinks,allcolors=blue,linktocpage=true}
\usepackage{url, float}  
\usepackage{amsfonts}       
\usepackage{amssymb}
\usepackage{nicefrac}
\usepackage{algorithm, algorithmic}
\usepackage{microtype}
\usepackage{graphicx}
\usepackage{booktabs} 
\usepackage{amsthm}
\usepackage{bbm, amsmath, mathtools}
\usepackage{thm-restate}

\usepackage{cleveref}

\theoremstyle{definition}
\newtheorem{definition}{Definition}
\newtheorem{lemma}{Lemma}
\newtheorem{obsv}{Observation}
\newtheorem{proposition}{Proposition}
\newtheorem{corollary}{Corollary}
\newtheorem{assumption}{Assumption}
\newtheorem{theorem}{Theorem}
\DeclareMathOperator*{\argmax}{arg\,max}
\DeclareMathOperator*{\argmin}{arg\,min}

\usepackage{pifont}

\newcommand{\bitem}{\begin{itemize}}
\newcommand{\eitem}{\end{itemize}}
\newcommand{\benum}{\begin{enumerate}}
\newcommand{\eenum}{\end{enumerate}}
\newcommand{\bdefn}{\begin{definition}}
\newcommand{\edefn}{\end{definition}}
\newcommand{\bprop}{\begin{proposition}}
\newcommand{\eprop}{\end{proposition}}
\newcommand{\bque}{\begin{question}}
\newcommand{\eque}{\end{question}}
\newcommand{\bobsv}{\begin{observation}}
\newcommand{\eobsv}{\end{observation}}
\newcommand{\beqn}{\begin{equation}\begin{aligned}}
\newcommand{\eeqn}{\end{aligned}\end{equation}}
\newcommand{\ps}{\begin{proof}[Sketch]}
\newcommand{\brmk}{\begin{remark}}
\newcommand{\ermk}{\end{remark}}
\newcommand{\bduiqi}{\begin{aligned}}
\newcommand{\eduiqi}{\end{aligned}}
\newcommand{\bcoro}{\begin{corollary}}
\newcommand{\ecoro}{\end{corollary}}
\newcommand{\bcom}{}

\newcommand{\adap}{adaptive}

\newcommand{\apxn}{approximation}

\newcommand{\arb}{arbitrary}

\newcommand{\alg}{algorithm}

\newcommand{\Alg}{Algorithm}

\newcommand{\assu}{assumption}

\newcommand{\bs}{\backslash}

\newcommand{\constr}{constraint}

\newcommand{\corres}{corresponding}

\newcommand{\ci}{confidence interval}

\newcommand{\distr}{distribution}

\newcommand{\dec}{decision}

\newcommand{\elimn}{elimination}

\newcommand{\elem}{element}

\newcommand{\feas}{feasible}

\newcommand{\func}{function}

\newcommand{\hypo}{hypothesis}
\newcommand{\hypos}{hypotheses}

\newcommand{\ho}{\mathbb}

\newcommand{\indep}{independent}

\newcommand{\ins}{instance}

\newcommand{\ineq}{inequality}
\newcommand{\ineqs}{inequalities}

\newcommand{\ift}{it follows that}

\newcommand{\lar}{\leftarrow}

\newcommand{\omg}{\omega}
\newcommand{\Omg}{\Omega}
\newcommand{\ord}{ordinary}
\newcommand{\obsvn}{observation}

\newcommand{\obj}{objective}
\newcommand{\OTOH}{On the other hand}

\newcommand{\pa}{partially-adaptive}

\newcommand{\prb}{probability}

\newcommand{\parti}{particular}

\newcommand{\pmt}{parameter}

\newcommand{\rv}{random variable}

\newcommand{\rar}{\rightarrow}

\newcommand{\real}{\mathbb{R}}

\newcommand{\sat}{satisfy}
\newcommand{\sats}{satisfies}
\newcommand{\satd}{satisfied}
\newcommand{\sce}{scenario}
\newcommand{\sep}{separat}

\newcommand{\sln}{solution}
\newcommand{\sps}{suppose}
\newcommand{\Sps}{Suppose}

\newcommand{\strfwd}{straightforward}

\newcommand{\submod}{submodular}

\newcommand{\Sg}{Subgaussian}
\newcommand{\sg}{subgaussian}

\newcommand{\unif}{uniform}

\newcommand{\xulie}{sequence}

\let\eps\varepsilon

\newcommand{\KL}{\mathrm{KL}}
\newcommand{\OPT}{\mathrm{OPT}}
\newcommand{\PA}{\mathrm{PA}}
\newcommand{\FA}{\mathrm{FA}}
\newcommand{\poly}{\mathrm{poly}}
\newcommand{\CT}{\mathrm{CT}}

\newcommand{\rev}{\color{black}}

\title{Greedy Approximation Algorithms \\ for Active Sequential Hypothesis Testing}
\author{
  Kyra Gan$^\ast$, Su Jia\thanks{Equal contribution.}, Andrew A. Li \\
  Carnegie Mellon University\\
  Pittsburgh, PA 15213 \\
  \texttt{\{\href{mailto:kyragan@cmu.edu}{kyragan},\href{mailto:sjia1@andrew.cmu.edu}{sujia},\href{mailto:aali1@cmu.edu}{aali1}\}@cmu.edu}
}

\begin{document}
\maketitle
\begin{abstract}
In the problem of \emph{active sequential hypothesis testing} (ASHT), a learner seeks to identify the \emph{true} hypothesis from among a known set of hypotheses. The learner is given a set of actions and knows the random distribution of the outcome of any action under any true hypothesis. 
Given a target error $\delta>0$, the goal is to sequentially select the fewest number of actions so as to identify the true hypothesis with probability at least $1 - \delta$. Motivated by applications in which the number of hypotheses or actions is massive (e.g., genomics-based cancer detection), 
we propose efficient (greedy, in fact) algorithms and provide the first approximation guarantees for ASHT, under two types of adaptivity.  
Both of our guarantees are independent of the number of actions and logarithmic in the number of hypotheses.
We numerically evaluate the performance of our algorithms using both synthetic and real-world DNA mutation data, demonstrating that our algorithms outperform previously proposed heuristic policies by large margins.
\vspace{-5pt}
\end{abstract}
\section{Introduction}
Consider the problem of learning the \emph{true} hypothesis from among a (potentially large) set of candidate hypotheses $H$. Assume that the learner is given a (potentially large) set of actions
$A$, and knows the distribution of the noisy outcome of each action, under each potential hypothesis. The learner incurs a fixed cost each time an action is selected, and seeks to identify the true hypothesis with sufficient confidence, at minimum total cost. Finally, and most importantly, the learner is allowed to select actions {\em adaptively}.

This well-studied problem is referred to as {\em active sequential hypothesis testing}, and as we will describe momentarily, there exists a broad set of results that tightly characterizes the optimal achievable cost under various notions of adaptivity. Unfortunately, the corresponding optimal policies are typically only characterized as the optimal policy to a Markov decision process (MDP)---thus,
they remain computationally hard to compute when one requires a policy in practice. This deficiency becomes particularly apparent in modern applications where both the set of hypotheses and set of actions are large. As a concrete example, we will describe later on an application to cancer blood testing that
has tens of hypotheses and {\em billions} of tests at full scale. 
Thus motivated, {\em we provide the first approximation algorithms for ASHT}. 

We study ASHT under two types of adaptivity: {\em partial} and {\em full}, where partial adaptivity requires the sequence of actions to be decided upfront (with adaptively chosen stopping time), and full adaptivity allows the choice of action to depend on previous outcomes.
For both problems, we propose {\em greedy} algorithms that run in $O(|A||H|)$ time, and prove that their expected costs are upper bounded by a non-trivial multiplicative factor of the corresponding optimal costs. Most notably, these approximation guarantees are {\em independent} of $|A|$ (contrast this with the trivially-achievable guarantee of $O(|A|)$) and {\em logarithmic} in $|H|$ (the optimal cost itself is often $\Omega(|H|)$).

\begin{table*}[t!]
\centering
\begin{tabular}{ccccc}
\toprule
& {\bf Noise} & 
{\bf Approximation Ratio}
& {\bf Objective} & {\bf Adaptivity Type} \\ \midrule
\cite{naghshvar2013active}  & Yes& No & Average & Both \\ 
\cite{Nowak09} & Yes   & No & Worst-case & Fully adaptive\\ 
\cite{im2016minimum}    & No & Yes & Both           & Partially adaptive \\ 
\cite{kosaraju1999optimal, chakaravarthy2009approximating}        & No & Yes & Both & Fully adaptive\\ 
\cite{jia2019optimal}   & Semi*  & No & Both & Both \\ 
This Work& Yes   & Yes & Both & Both\\ \bottomrule
\end{tabular}
\vspace{-5pt}
\caption{ Summary of related work.
*\emph{Semi} refers to a restrictive special case.
}
\label{table:survey}
\vspace{-10pt}
\end{table*}

Our results rely on drawing connections to two existing problems: {\em submodular function ranking} (SFR) \cite{azar2011ranking} and the {\em optimal decision tree} (ODT) problem \cite{laurent1976constructing}. These connections allow us to tackle what is arguably the primary challenge in achieving approximation results for ASHT, which is its inherent {\em combinatorial} nature. We will argue that existing heuristics from statistical learning fail precisely because they disregard this combinatorial difficulty---indeed, they largely amount to solving the completely {\em non-adaptive} version of the problem.
At the same time, existing results for SFR and ODT fail to account for {\em noise} in a manner that would map directly to ASHT---this extension is among our contributions. 

\paragraph{Related Work }
Our work is closely related to three streams of research.
Table~\ref{table:survey} highlights the key differences between our contributions and those of the most relevant previous works.
\begin{enumerate}
    \item[(a)] \textbf{Hypothesis Testing and Asymptotic Performance:}
In the classical binary sequential hypothesis testing problem, a decision
maker is provided with one action 
whose outcome
is stochastic \cite{wald1945sequential, armitage1950sequential, lorden1977nearly}, and the goal is to use the minimum expected number of samples to identify the true hypothesis subject to some given error probability. 
The ASHT problem, first studied in~\cite{chernoff1959sequential}, generalizes this problem to multiple actions. Most related to our work is~\cite{naghshvar2013active}, who formulated a similar problem as an MDP. We will postpone describing and contrasting their work until the experiments section. 
%
%
%
%
%
\item[(b)]\textbf{Active Learning and Sample Complexity:}
In active learning, the learner is given access to a pool of unlabeled samples (cheaply obtainable) and is allowed to request the label of any sample (expensive) from that  pool. 
The goal is to learn an accurate classifier while requesting as few labels as possible.
Some nice surveys include~\cite{hanneke2014theory} and~\cite{settles2009active}.
Our model extends the classical discrete active learning model~\cite{dasgupta2005analysis} in which outcomes are noiseless (deterministic) for any pair of \hypo\ and unlabeled sample.
When outcomes are noisy, the majority of provable guarantees are provided via sample complexity.
%
%
%
\cite{castro2007minimax} showed tight minimax classification error rates for a broad class of distributions. Other sample complexity results on noisy active learning include~\cite{wang2014noise,Nowak09,BalcanBL06,awasthi2017power,hanneke2015minimax}. 
\item[(C)]\textbf{ Approximation Algorithms for Decision Trees:} 
Nearly all optimal \apxn\ \alg s for minimizing cover time are known in the noiseless setting \cite{kosaraju1999optimal, adler2008approximating,arkin1998decision}. 
When the outcome is stochastic, \cite{GolovinK11} proposed a framework for analyzing algorithms 
under the \emph{adaptive submodularity} assumption.
However, their  \assu\ does not hold for many natural setups including ASHT. \cite{chen2015sequential} considered 
a variant using ideas from the submodular max-coverage problem,
and provided a constant factor \apxn\ to the 
problem. 
Other works based on submodular function covering include \cite{navidi2020adaptive,guillory2010interactive,krause2008robust}.
\cite{jia2019optimal} provided \apxn\ ratios under the
\constr\ that the \alg\ may only terminate when it is completely confident about the outcome.

\end{enumerate}
\section{Model}
\label{sec:model}
We begin by formally introducing the problem.
Let $H$ be a finite set of {\em \hypos}, among which exactly one is the (unknown) {\em true} hypothesis that we seek to identify. 
In this paper, we study the {\em Bayesian} setting, wherein this true \hypo\ is drawn from a known prior \distr\ $\pi$ over $H$.

Let $A$ be the set of available {\em{actions}}. 
Selecting an action yields a random {\em outcome} drawn independently from a distribution within a given family $\mathcal{D} = \{D_\theta\}_{\theta \in \Theta}$ of distributions parameterized by $\Theta \subseteq \mathbb{R}$. 
We are given a function $\mu: H \times A \to \Theta$ 
such that 
if $h \in H$ is the underlying \hypo\ and we select action $a\in A$, then the random outcome is drawn independently from distribution $D_{\mu(h,a)}$.\footnote{In this noisy setting, an action can (and often should) be played for multiple times.}

An {\em{instance}} of the active sequential hypothesis testing problem is then fully specified by a tuple: $(H,A,\pi, \mu, \cal{D})$.
The goal is to sequentially select actions to identify the true \hypo\ with ``sufficiently high'' confidence, at minimal expected cost, where cost is measured as the number of actions, 
and the expectation is with respect to the Bayesian prior and the random outcomes. 
The notion of \emph{sufficiently high} confidence is encoded by a parameter $\delta \in (0,1)$, and requires that under any true $h \in H$, the probability of erroneously identifying a different hypothesis is at most $\delta$. An algorithm which satisfies this is said to have achieved {\bf $\delta$-PAC-error}.

We focus on two important families of $D_\theta$'s: the Bernoulli \distr\ $\mathrm{Ber}(\theta)$ and the Gaussian \distr\ $N(\theta, \sigma^2)$ where $\sigma^2$ is a known constant (with respect to $\theta$).
By re-scaling, without loss of generality we may assume $\sigma^2 =1$.
We require two additional assumptions to state our guarantees. 
The first assumption
{\rev is needed for relating the sub-gaussian norm to the KL-divergence, in the partially adaptive version. It}
ensures that the parameterization $\Theta$ is a meaningful one, in the sense that if $\theta,\theta' \in \Theta$ are far apart, then the distributions $D_\theta$ and $D_{\theta'}$ are also ``far'' apart (as measured by KL divergence).
Assumption~\ref{assumption:KL} is \satd\ for 
$\mathrm{Ber}(\theta)$ when $ \theta\in [\theta_{\min}, \theta_{\max}]$ for some constants $0<\theta_{\min}<\theta_{\max}<1$, and for $N(\theta,1)$ where $\theta$ lies in some bounded set in $\real$.
\begin{assumption}\label{assumption:KL}
There exist
$C_1, C_2>0$ such that for any $ \theta,\theta'\in\Theta$, we have 
$C_1\cdot \KL(D_\theta, D_{\theta'}) \leq (\theta-\theta')^2 \leq C_2\cdot \KL(D_\theta, D_{\theta'}),$
where $\KL(\cdot,\cdot)$ is the Kullback-Leibler divergence.
\end{assumption}
%



Our second major assumption simply ensures the existence of a valid algorithm, by assuming that every hypothesis is distinguishable via some action.
\begin{assumption}[Validity]\label{assu1}
For all $g,h\in H$ where $g\neq h$, there exists $a\in A$ with $\mu(g,a) \ne \mu(h,a)$.
\end{assumption}
In \parti, we do not preclude the possibility that for a given action $a$, there exist (potentially many) pairs of hypotheses $g,h$ such that $\mu(g,a) = \mu(h,a)$. In fact, eliminating such possibilities would effectively wash out any meaningful combinatorial dimension to this problem. On the other hand, any approximation guarantee should be parameterized by some notion of separation (when it exists). 
For any two hypotheses $g,h \in H$ and any action $a \in A$, define $d(g,h;a) = \KL\left(D_{\mu(g,a)},D_{\mu(h,a)}\right).$

\begin{definition}[$s$-\sep ed instance]
An ASHT \ins\ is said to be {\bf{$s$-\sep ed}}, if for any $a\in A$ and $g,h\in H$, $d(g,h;a)$ is either $0$ or at least $s$.
\end{definition}
%
{\rev Note that in real-world applications, the parameter $s$ could be arbitrarily small, and we introduce the notion of s-separability for the sake of proofs. We will show in Section~\ref{sec:experiments} how our algorithms can easily be modified to handle small $s$ values. In this work,}
we will study two classes of algorithms that differ in the extent to which adaptivity is allowed. 
%
%
%
%
%
%



\begin{definition}
A {\bf{fully \adap}} algorithm is a decision tree,\footnote{By approximating $D_\theta$'s with discrete \distr s, we may assume each node has a finite number of children.} each of whose interior nodes is labeled with some action, and each of whose edges corresponds to an outcome. 
Each leaf is labeled with a \hypo, corresponding to the output when the algorithm terminates.
\end{definition}

\begin{definition}
A {\bf{partially \adap}} algorithm $(\sigma,T)$ is specified by a fixed sequence of actions $\sigma = (\sigma_1,\sigma_2,...)$, with each $\sigma_i\in A$, and a {\it stopping} time $T$. In particular, 
under any true hypothesis $h^* \in H$ and for any $t\geq 1$,
the event
$\{T = t\}$ is independent of the outcomes of actions $\sigma_{t+1}, \sigma_{t+2},\hdots$ (At the stopping time, the choice of which hypothesis to identify is trivial in our Bayesian setting---it is simply the one with the highest ``posterior'' probability).
\end{definition}

Note that a partially \adap\ algorithm can be viewed as a special type of fully \adap\ algorithm: it is a decision tree with the additional restriction that the actions at each depth are the same. 
Therefore, a fully \adap\ \alg\ may be far cheaper than any partially \adap\ \alg. However, there are many \sce s (e.g., content recommendation and web search~\cite{azar2009multiple}) where it is desirable to fix the \xulie\ of actions in advance. Furthermore, in many problems the theoretical analysis of partially \adap\ \alg s turns out to be challenging (e.g., \cite{kamath2019anaconda,chawla2019learning}).

%
%
%
Thus, given an ASHT instance, 
there are two problems that we will consider, depending on whether the algorithms are partially or fully \adap. 
In both cases, our goal is to design fast approximation algorithms---ones that are computable in polynomial\footnote{Throughout this paper, \emph{polynomial time} refers to polynomial in $\left(|H|,|A|,s^{-1},\delta^{-1}\right)$}  time and that are guaranteed to incur expected costs at most within a multiplicative factor of the optimum. In the coming sections, we will describe our algorithms and approximation guarantees. Before moving on to this, it is worth noting that our problem setup is extremely generic and captures a number of well-known problems related to decision-making for learning including best-arm identification for multi-armed bandits \cite{bubeck2009pure,even2002pac,mannor2004sample}, group testing \cite{du2000combinatorial}, and causal inference \cite{gan2020causal}, just to name a few.




\section{Our Approximation Guarantees}
\label{sec:guarantees}
We are now prepared to state our approximation guarantees (the corresponding greedy algorithms will be defined in the next two sections).
Let $\OPT^{\PA}_\delta$ (resp. $\OPT^{\FA}_\delta$) denote the minimal expected cost of any partially \adap\ (resp. fully \adap) \alg\ that achieves $\delta$-PAC-error. 

\begin{restatable}{theorem}{thmPA}
\label{thm:pa}
Given an $s$-\sep ed \ins\ and any $\delta\in (0,1/2)$, 
there exists a polynomial-time partially \adap\ algorithm that achieves $\delta$-PAC-error with expected cost $O\left(s^{-1}\left(1+\log_{1/\delta}|H|\right)\log \left(s^{-1}|H| \log \delta^{-1}\right)\right) \OPT^\PA_\delta.$
\end{restatable}
To help parse this result, if $\delta$ is on the order of $|H|^{-c}$ for some constant $c$, then the \apxn\ factor becomes $s^{-1}(\log s^{-1}+\log |H| + c\log\log |H|)$.

\begin{restatable}{theorem}{thmODTpac}
\label{thm:odt_pac}
Given an $s$-\sep ed instance  and any $\delta\in (0,1/2)$, there exists a polynomial-time
 fully \adap\ algorithm that achieves
$\delta$-PAC-error with expected cost $O\left(s^{-1} \log \left(|H|\delta^{-1}\right) \log |H|\right)  \OPT^\FA_\delta.$ 
\end{restatable}

A few observations might clarify the significance of these approximation guarantees: 
\begin{enumerate}
\item Dependence on action space: Both guarantees are independent of the number of actions $|A|$. This is extremely important since, as described in the Introduction, there exist many applications where the the action space is massive. Moreover, since an approximation factor of $O(|A|)$ is always trivially achievable (by cycling through the actions), instances where $|A|$ is large are arguably the most interesting problems.
\item Dependence on $|H|, \delta$ and $s$: For fixed $s$ and $\delta$, these are the first polylog-\apxn s for 
both partially and fully adaptive 
versions. 
Further, for the partially adaptive version,
the dependence of the
\apxn\ factor on $\delta$ is
$O(\log\log \delta^{-1})$ when $\delta^{-1}$ is polynomial in $|H|$, improving upon the naive dependence $O(\log \delta^{-1})$.
This is crucial since $
\delta$ is often needed to be tiny in practice.
\item Greedy runtime: While we have only stated in our formal results that our approximation algorithms can be computed in $\poly(|A|,|H|)$ time, the actual time is more attractive: $O(|A||H|)$ for selecting each action.
In contrast, the heuristic that we will compare against in the experiments requires solving multiple $\Omega(|A||H|^2)$-sized linear programs.
\end{enumerate}

Despite their similar appearances, Theorems \ref{thm:pa} and \ref{thm:odt_pac} rely on fundamentally different algorithmic techniques and thus require different analyses. In Section \ref{sec:partial}, we propose an \alg\ inspired by the \emph{submodular function ranking} problem, which greedily chooses a sequence of actions according to a carefully chosen ``greedy score.'' We then sketch the proof of 
Theorem~\ref{thm:pa}. 
In Section \ref{sec:fa},
we introduce our fully adaptive algorithm and sketch the proof of Theorem~\ref{thm:odt_pac}.
%

Finally, by proving a structural lemma (in Appendix \ref{append:total_error}), we extend the above results to a special case of the {\bf total-error} version (i.e., averaging the error over the prior $\pi$) where the prior is uniform.
With \emph{$\delta$-total-error} formally defined in Appendix~\ref{append:total_error}:

\begin{restatable}{theorem}{thmODT}
\label{thm:odt}
Given an $s$-\sep ed instance with uniform prior $\pi$ and any $\delta \in (0,1/4)$, for
both the partially and fully adaptive
versions, there exist  polynomial-time  
$\delta$-total-error algorithms with expected cost $O\left(s^{-1}\left(1+|H|\delta^2\right)\log \left(|H|\delta^{-1}\right)\log |H| \right)$ times the optimum.
\end{restatable}

\section{Partially Adaptive \Alg}\label{sec:partial}
This section describes our algorithm and guarantee for the partially adaptive problem.
We first 
review necessary background from a related problem, and then state our algorithm (Algorithm \ref{alg:rnb}). Finally,
we sketch the proof of the following more general version of Theorem~\ref{thm:pa} (complete proof in Appendix~\ref{sec:pf_pa}):
\begin{proposition}\label{prop:pa}
Let $\delta\in (0, \frac 14]$ and consider finding the optimal $\delta$-PAC error \alg.
Given any boosting intensity $\alpha\geq 1$ and  coverage saturation threshold $B\in (0, \frac 12 \log\delta^{-1}]$, $\mathrm{RnB}(B,\alpha)$ (as defined in Algorithm \ref{alg:rnb}) produces a partially \adap\ \alg\ with error $|H|\exp\left(-\Omg\left(\alpha B\right)\right)$ and expected cost $O\left(\alpha s^{-1} \log \left( |H|B s^{-1} \right)\right)\mathrm{OPT}^{\PA}_\delta$.
\end{proposition}
By setting $\alpha=1+\log_{\delta^{-1}}|H|$ and $B=\frac 12 \log \delta^{-1}$, we immediately obtain Theorem~\ref{thm:pa}.

\paragraph{Background: Submodular Function Ranking }
%
In the SFR problem,
we are given a ground set $U$ of $N$ \emph{\elem s}, a family $\mathcal{F}$ of non-decreasing submodular functions $f: 2^U \rar [0,1]$ with $f(U)$ equaling $1$ for every $f \in \mathcal{F}$, and a weight \func\ $w:\mathcal{F}\rar \real^+$. 
For any permutation $\sigma = (u_1,...,u_N)$ of $U$,
the {\emph{cover time}} of $f$ is defined as $\mathrm{CT}(f, \sigma) = \min\{t: f(\{u_1,...,u_t\})=1 \}$. 
The goal is to find a permutation $\sigma$ of $U$ with minimal {\it cover time} $\sum_{f\in \mathcal{F}} w(f)\cdot \mathrm{CT}(f, \sigma).$
We will use the following greedy algorithm, called GRE, in \cite{azar2011ranking} as a subroutine. 
The sequence is initialized to be empty and is constructed iteratively.
At each iteration, let $S$ be the \elem s selected so far.
GRE selects the \elem\ $u$ with the maximal {\it coverage}, defined as 
$\mathrm{Cov}(u; S):= \sum_{f\in \mathcal{F}: f(S)<1} w(f)\cdot \left(f(S\cup \{u\})-f(S)\right)
/\left(1-f(S)\right).$

\begin{theorem}[\cite{im2016minimum}]\label{thm:inz}
For any SFR \ins, GRE returns a \xulie\ whose cost is $O(\log  \eps^{-1})$ times the optimum, where $\eps:=\min\left\{f(S\cup \{u\}) - f(S) > 0: S\in 2^U, u\in U, f\in \mathcal{F}\right\}$.
\end{theorem}
\paragraph{Challenge } To motivate our algorithm, consider first the following simple idea: ``boost'' ({\rev or repeat}) each action {\rev enough}, and hence reduce the problem to a deterministic problem $P_{det}$.
We then show that the existing technique (submodular function ranking for partially \adap\ and greedy analysis for ODT for fully-\adap) returns a policy with cost $O(\log |H|)$ times the no-noise optimum, and finally show that this no-noise policy can be converted to a noisy version by losing anther factor of $O(s^{-1}\log (\delta^{-1}|H|))$. 
This analysis was in fact our first attempt.
However, there are at least two issues that one runs into:
\begin{enumerate}
    \item This analysis only compares the policy's cost with the no-noise optimum, but our focus is the $\delta$-noise optimum.
In particular, the simpler analysis implicitly assumes that the $\delta$-noise optimum is at least $\Omega(s^{-1}\log (\delta^{-1}|H|))$ times the no-noise optimum, which is not necessarily true.
Moreover, it is challenging to analyze the gap between the no-noise optimum and the $\delta$-noise optimum.
\item 
This simple analysis provides a \emph{weaker} guarantee than ours
in terms of $\delta$: it yields a factor of $\log (1 / \delta)$, as opposed to the $\log \log (1 /\delta)$ in our analysis.
This distinction is nontrivial, particularly in applications where the error is required to be exponentially small in $|H|$. 
\end{enumerate}

\paragraph{Rank and Boost (RnB) \Alg }
Our RnB \alg\ (Algorithm~\ref{alg:rnb}) circumvents the issues above by drawing a connection between ASHT and SFR. 
First, we observe that
although an action is allowed to be selected for multiple times, we may assume each action is selected for at most $M = M(\delta,s,|H|) = O(s^{-1}|H|^2 \log (|H|/\delta))$ times. In fact,
\begin{obsv}
Let $\widetilde A$ be the (multi)-set obtained by creating $M$ copies of each $a\in A$. 
Then there exists a \xulie\ $\sigma$ of $|\widetilde A|$ actions, s.t. for any true \hypo\ $h^*\in H$, $h^*$ has the highest posterior with \prb\ $1-\delta$ after performing all actions in $\sigma$.
\end{obsv}
Thus, given $\widetilde A$, we define $f_h^B: 2^{\widetilde A} \rar [0,1]$
for any 
coverage saturation level 
$B>0$ and $h\in H$
as
$f_h^{B}(S) = (|H|-1)^{-1} \sum_{g\in H\bs \{h\}} \min\{1,
B^{-1} \sum_{a\in S} d(g,h;a)
\}$.
One can verify that 
$f^B_h$ is monotone and \submod.
Our \alg\ computes a nearly optimal \xulie\ of actions using the greedy \alg\ for SFR, and creates a number of copies for each of them. 
Then we assign a \emph{timestamp} to each $h\in H$, and scan them one by one, terminating when the likelihood of one \hypo\ is dominantly high. 
\begin{algorithm}[t]
\caption{{\bf Partially Adaptive Algorithm: $\mathrm{RnB}(B,\alpha)$}}
\begin{algorithmic}[1]
\label{alg:rnb}
\STATE{\textbf{Parameters}: Coverage saturation level $B>0$ and boosting intensity $\alpha>0$.}
\STATE{\textbf{Input}: ASHT instance $(H,A,\pi,\mu, \cal{D})$ }
\STATE{\textbf{Initialize}: $\sigma\lar\emptyset, \tilde \sigma\lar\emptyset$} \quad\quad\quad\quad\quad\quad\quad\quad\quad\quad\quad\quad\quad\quad\quad \% Store the selected of actions.
\STATE{\textbf{For}}  {$t=1,2,..., |\widetilde A|$} \textbf{do} \quad\quad\quad\quad\quad\quad\quad\quad\quad\quad\quad \% {\bf Rank:} Compute a \xulie\ of actions.
\STATE{\quad $S\lar \{\sigma(1),...,\sigma(t-1)\}$.} \quad\quad\quad\quad\quad\quad\quad\quad\quad\quad\quad\quad\quad\quad\quad \% Actions selected so far.
\STATE{\quad {\bf For} $a\in \widetilde A$,  
\quad\quad\quad\quad\quad\quad\quad\quad\quad\quad\quad\quad\quad\quad\quad\quad\quad\quad \% Compute scores for each action.
\[\mathrm{Score}(a;S) \lar \sum_{h: f^B_h(S)<1} \pi(h) \frac{f^B_h(S\cup \{a\})- f^B_h(S)}{1-f^B_h(S)}.\]
}
\STATE{
\quad $\sigma(t) \lar \argmax\{ \mathrm{Score}(a;S): a\in \widetilde A\bs S\}.$
} \quad\quad\quad\quad\quad\quad\quad\ \% Select the greediest action.
\STATE{{\bf For} $t=1,2,...,|\tilde A|$:}
\quad\quad\quad\quad\quad\quad\quad\quad\quad\quad\quad \% {\bf Boost:} Repeat each action in $\sigma$ for $\alpha$ times.
\STATE{\quad {\bf For} $i=1,2,...,\alpha$:}
\STATE{\quad\quad $\tilde \sigma\big(\alpha (t-1)+ i \big) \lar \sigma(t)$.}
\STATE{{\bf For} $t=1,...,\alpha |\tilde A|$:}
\STATE{\quad Select action $\tilde \sigma(t)$ and observe outcome $y_t$.}
\STATE{\quad {\bf If} $t=\alpha\cdot \mathrm{CT}(f^B_h, \sigma)$ for some $h\in H$:}
\quad\quad\quad\quad\quad\quad\quad \% If $t$ is the {\it timestamp} for some $h$.
\STATE{\quad\quad {\bf For} $g\in H\bs \{h\}$:}
\STATE{\quad\quad\quad $\Lambda(h,g) \lar \prod_{i=1}^{t} \ho{P}_{h,\tilde \sigma(i)}(y_i)/\ho{P}_{g,\tilde \sigma(i)} (y_i)$.}
\quad\quad\quad\quad\quad \% Compute the likelihood ratio.
\STATE{\quad\quad {\bf If} $\log \Lambda(h,g)\geq \alpha B/2$ for all $g\in H\bs \{h\}$, {\bf then} Return $h$.} \quad\quad\quad \% Hypothesis identified.
\end{algorithmic}
\end{algorithm}

\noindent{\rev Although a naive implementation of Algorithm~\ref{alg:rnb} yields a running time that is linear in the number of actions, however since Score$(a; S)$ (Line 6 of Algorithm~\ref{alg:rnb}) can be calculated independently for each action $a$, one could paralyze this calculation for different actions and thus reducing the dependency on $|A|$. The same observation also holds for the rest algorithms to be introduced in the paper.}

\paragraph{Proof Sketch for Proposition~\ref{prop:pa} }
We sketch a proof and defer the details to Appendix~\ref{sec:pf_pa}.
The error analysis follows from standard concentration bounds, so we focus on the cost analysis.
\Sps\ $\alpha>0$, $\delta\in (0,1/4]$, and $B\in (0, (1/2)\log\delta^{-1}]$. 
Let $(\sigma^*,T^*)$ be any optimal partially adaptive algorithm, and let $(\sigma,T)$ be the policy returned by RnB. 
Our analysis consists of the following steps:
\benum
\item[(A)] The \xulie\ $\sigma$ does well in covering the \submod\ functions, in terms of the total cover time:
$\sum_{h\in H} \pi(h)\cdot \mathrm{CT}(f^B_h, \sigma) \leq 
O\left(\log \left(|H|B s^{-1}\right)\right) \sum_{h\in H} \pi(h)\cdot  \mathrm{CT}(f^B_h, \sigma^*).$
\item[(B)] The expected stopping
time of our \alg\ is not too much higher than the cover time of its submodular function:
$\ho{E}_{h} [T] \leq \alpha \cdot \mathrm{CT}(f^B_h, \sigma), \; \forall h\in H.
$
\item[(C)] 
The expected stopping time in $(\sigma^*,T^*)$ can be lower bounded in terms of the total cover time:
$\ho{E}_h [T^*] \geq \Omg(s)\cdot \mathrm{CT}(f^B_h, \sigma^*), \; \forall h\in H.
$
\eenum
Proposition~\ref{prop:pa} follows by combining the above three steps. In fact,
\beqn
\sum_{h\in H}  \pi(h) \cdot \ho{E}_h [T] &\leq \alpha\sum_{h\in H} \pi(h)\cdot \CT(f^B_h, \sigma) 
\leq O\left(\alpha\log\frac{|H|B}{s} \right) \sum_{h\in H} \pi(h) \cdot \CT(f^B_h, \sigma^*)\\
&\leq O\left(\frac{\alpha}{s}\log\frac{|H|B}{s} \right) \sum_{h\in H} \pi(h) \cdot \ho{E}_h [T^*],\notag
\eeqn
where $\pi(h) \cdot \ho{E}_h [T]$ is the expected cost of our algorithm, and $\pi(h) \cdot \ho{E}_h [T^*]$ is the expected cost of the optimal partially adaptive algorithm, $\OPT_{\delta}^{\PA}$.

At a high level, Step A can be showed by applying Theorem~\ref{thm:inz} and observing that the marginal positive increment of each $f_h^B$ is $\Omg({s}/({|H|B)})$.
Step B is implied by the correctness of the \alg.
In our key step, Step C, we fix an \arb\ $\delta$-PAC-error partially \adap\ \alg\ $(\sigma,T)$ and $h\in H$.
Denote $\CT_h=CT(f_h^{B}, \sigma)$, with $B$ chosen to be $\frac 12 \log \delta^{-1}$.
Our goal is to lower bound $\ho{E}_h [T]$ in terms of 
$\mathrm{CT}_h$.
To this aim, we consider an LP. 
Given any $d_1,...,d_n$, denote $d^i=\sum_{j=1}^i d_j$. 
Define
\beqn
LP(d, t):\quad \left\{\min_z
\sum_{i=1}^N i \cdot z_i \Bigg\rvert
\sum_{i=1}^N d^i z_i \geq \sum_{i=1}^{\CT_h-1} d_i,
\sum_{i=1}^N z_i = 1, 
z\geq 0.\right\}
\notag
\eeqn
A \feas\ \sln\ $z$ can be viewed as a \distr\ of the stopping time. 
When $d_i=d(g,h;a_i)$, the first \constr\ says that the total KL-divergence ``collected'' at the stopping time has to reach a certain threshold.
We show that $z_i = \ho{P}_h [T=i]$ is feasible, and the \obj\ value of $z$ is exactly $\ho{E}_h [T]$. Hence $\ho{E}_h [T]$ is upper bounded by the LP-optimum $LP^*(d,t)$.
Finally, we lower bound $LP^*(d, \mathrm{CT}_h-1)$ by $\Omg(s \cdot \mathrm{CT}_h)$, and the proof follows.
\section{Fully Adaptive \Alg}\label{sec:fa}
For ease of presentation we only consider the uniform prior version here (though our guarantees do hold for general priors). Our analysis is based on a reduction to the classical ODT problem.

\paragraph{Background: Optimal Decision Trees}
In the ODT problem, an \emph{unknown} true hypothesis $h^*$ is drawn from a set of hypotheses $H$ with some known probability distribution $\pi$. 
There is a set of known \emph{tests}, each being a  (deterministic) mapping from $H$ to a finite {\it outcome space} set $O$. 
Thus, when performing a test, we can \emph{rule out} the \hypos\ that are inconsistent with the observed outcome, hence reducing the number of \emph{alive} hypotheses.
Moreover, the cost $c(T)$ of each test $T$ is known, and the \emph{cost of a decision tree} is defined to be the expected total cost of the tests selected until one \hypo\ remains \emph{alive}, in which case we say the true \hypo\ is \emph{identified}. 
The goal is to find a valid decision tree with minimal expected cost.

Note that the ODT problem can be viewed as a special case of the fully \adap\ version of our problem where there is no noise and $\delta$ is 0.
Consider the following greedy algorithm GRE: let $A$ be the alive \hypos.
Define $\mathrm{Score(T)}$ for each test $T$ to be the minimal (over all possible outcomes) number of alive \hypos\ that it rules out in $A$.
Then, we select the test $T$ with the highest ``bang-per-buck'' $\mathrm{Score}(T)/c(T)$.
This \alg\ is known to be an $O(\log |H|)$-\apxn.
\begin{theorem}[\cite{chakaravarthy2009approximating}]\label{thm:modt}
For any ODT \ins\ with uniform prior, GRE returns a \dec\ tree whose cost is $O(\log |H|)$ times the optimum.
\end{theorem}
\paragraph{Our Algorithm }
We will analyze our 
greedy \alg\ by relating to the above result.
Consider the following ODT \ins\ $\mathcal{I}_{\mathrm{ODT}}$ for any given ASHT \ins\ $\mathcal{I}$.
The \hypos\ set and prior in $\mathcal{I}_{\mathrm{ODT}}$ are the same as in $\mathcal{I}$. 
For each action $a\in A$, 
let $\Omg_a:=\{\mu(h,a)|h\in H\}$ be the mean outcomes. 
By Chernoff bound, we can show that when $h$ is the true \hypo, with high probability the mean outcome is ``close'' to $\mu(h,a)$ when $a$ is repeated for $c(a)$ times.
This motivates us to define a test $T_a: H \rar \Omg_a$ s.t. $T_a(h)= \mu(h,a)$, with cost 
$c(a) = \lceil s(a)^{-1}\log (|H|/\delta) \rceil$,
where $s(a) = \min \{d(g,h;a)>0: g,h\in H\}$ is the separation parameter under action $a$.
Such a test corresponds to selecting $a$ for $c(a)$ times in a row.

For each $\omg\in \Omg_a$, 
abusing the notation a bit,
let $T_a^\omg\subseteq H$ denote the set of \hypos\ whose outcome is $\omg$ when performing $T_a$, i.e., $T_a^\omg=\{h: \mu(h,a)=\omg\}$.
At each step, Algorithm \ref{alg:fa} selects an action $\hat a$ using the greedy rule (Step~\ref{step:greedy_fa}) and then repeat $\hat a$ for $c(\hat a)$ times.
Then we round the empirical mean of the \obsvn s to the closest \elem\ $\hat\omega$ in $\Omega_a$, ruling out inconsistent hypotheses,
i.e., the $h$'s with $\mu(h,a)\neq \hat\omg$.
We terminate when only one hypothesis remains alive.

\paragraph{Analysis }
We sketch a proof for Theorem~\ref{thm:odt_pac} and defer the details to Appendix~\ref{proof:prop2}.
Let $h^*$ be the true \hypo. 
By Hoeffding's \ineq, in each iteration, with probability $1-e^{-\log (|H|/\delta)} = 1- \delta/ |H|$ it holds $\hat \omg = \mu(h^*,\hat a)$.
Since in each iteration, $|H|$ decreases by at least 1, there are at most $|H|-1$ iterations.
Thus by union bound, the total error is at most $\delta$.

Next we analyze the cost.
Let GRE be the cost of Algorithm~\ref{alg:fa} and $\mathrm{ODT}^*$ be the optimum of the ODT instance $\mathcal{I}_{\mathrm{ODT}}$.
{\rev For the sake of analysis, we consider a ``fake'' cost $c' := \lceil s^{-1} \log (|H|/\delta) \rceil$, which does not depend on $a$.
The definition of the ODT instance $I_{ODT}$ remains the same except that each test has {\bf uniform} cost $c'$ (as opposed to $c(a)$).
Let $c(T)$ and $c'(T)$ be the costs of the greedy tree $T$ returned by Algorithm 2 under $c$ and $c'$ respectively.
Then by Theorem~\ref{thm:modt},
$c'(T) \leq O(\log |H|)\cdot \mathrm{ODT}^*.$
Note that $c'\leq c(a)$ for each $a$ since the separation parameter $s$ is no larger than $s(a)$ by definition.
Hence, 
\begin{align}\label{eqn:feb4}
\mathrm{GRE} = c(T)\leq c'(T)\leq O(\log |H|)\cdot \mathrm{ODT}^*.
\end{align}}
\noindent{We} relate $\mathrm{ODT}^*$ to $\OPT^{FA}_\delta$ using the following result (see proof in Appendix~\ref{proof:prop2}): 
\begin{proposition}\label{prop:convert}
$\mathrm{ODT}^*\leq O(s^{-1}\log ({|H|}/{\delta}))\cdot \OPT^{FA}_\delta$. 
\end{proposition}
The above is established by showing how to convert a $\delta$-PAC-error fully adaptive \alg\ to a valid \dec\ tree, using only tests in $\{T_a\}$, and inflating the cost by a factor of $O(s^{-1}\log (|H|/\delta))$. 
Combining Proposition~\ref{prop:convert} with Equation (\ref{eqn:feb4}), we obtain
$GRE\leq O(s^{-1} \log \frac{|H|}{\delta}\log |H|)\cdot  \OPT^{FA}_\delta.$

Finally we remark that this analysis can easily be extended to general priors by reduction to the \emph{adaptive submodular ranking} (ASR) problem
\cite{navidi2020adaptive}, which captures ODT as a special case.
{\rev One may easily verify that the main theorem in \cite{navidi2020adaptive} implies that a (slightly different) greedy algorithm achieves $O(\log(|H|))$-approximation for the ODT problem with general prior, test costs, and an arbitrary number of branches in each test. 
Thus for general prior, the same analysis goes through if we first reduce ASHT to ASR, and then replace the greedy step (Step~\ref{step:greedy_fa} in Algorithm~\ref{alg:fa}) with the greedy criterion for ASR. }


\begin{algorithm}[t]
\caption{\bf{Fully Adaptive \Alg}}
\begin{algorithmic}[1]
\label{alg:fa}
\STATE{{\bf Input:} ASHT \ins\ $(H,A,\pi,\mu, \cal{D})$} and error $\delta\in (0,1/2)$.
\STATE{$H_{\mathrm{alive}}\leftarrow H$.} 
\quad\quad\quad\quad\quad\quad\quad\quad\quad\quad\quad\quad\quad\quad\quad\quad\quad\quad\quad\quad\quad\quad\quad\quad \% {\it Alive} \hypos.
\WHILE{$|H_{\mathrm{alive}}| \geq 2$} 
\STATE{$\widehat a \lar \argmax_{a\in A} \Big\{\min_{\omg\in \Omg_a} |H_{\mathrm{alive}}\bs T_a^\omg|\Big\}$.}\label{step:greedy_fa} \quad\quad\quad\quad\quad\quad\quad\quad\quad\quad\quad\quad \% Greedy step.
\STATE{$c(\hat{a}) \lar \lceil s(\hat{a})^{-1}\log (|H|/\delta) \rceil$.} 
\quad\quad\quad\quad\quad\quad \% \# times to boost for sufficient confidence.
\STATE{Select $\hat a$ for $c(\hat{a})$ times consecutively and observe outcomes $X_1,...,X_{c(\hat a)}$.}
\STATE{$\hat \mu \lar \sum_{i=1}^{c(\hat a)} X_i$.} \quad\quad\quad\quad\quad\quad\quad\quad\quad\quad\quad\quad\quad\quad\quad\quad\quad\quad\quad\quad\quad\quad \% Mean outcome.
\STATE{$\hat \omg \lar \argmin \{|\hat \mu - \omg|: \omg\in \Omg_a\}$}. \quad\quad\quad\quad\quad\quad\quad\quad\quad\quad\quad \% Round $\hat\mu$ to the closest $\omg$.
\STATE{$H_{\mathrm{alive}}\lar H_{\mathrm{alive}}\cap T_{\hat a}^{\hat \omg}.$} 
\quad\quad\quad\quad\quad\quad\quad\quad\quad\quad\quad\quad\quad\quad \% Update the alive \hypos.
\ENDWHILE
\end{algorithmic}
\end{algorithm}

\vspace{-3pt}
\section{Experiments}
\label{sec:experiments}
\vspace{-3pt}
Although our theoretic guarantees depend on the separability parameters $s$ (which was introduced by the boosting steps), in this section, we numerically
demonstrate that with small modifications our algorithms perform well when $s$ is small on both synthetic and real-world data.
Our primary benchmarks are a polynomial-time policy proposed by \cite{naghshvar2013active} (\emph{Policy 1}\footnote{\emph{Policy 2} in \cite{naghshvar2013active} does not have asymptotic guarantees and so is not considered in our experiments.}) and a completely random policy.
{\rev To our knowledge, the policy proposed by \cite{naghshvar2013active} is the state-of-art algorithm (with theoretical guarantees) that can be applied to our problem setup.}
%
%
The rest of this section is organized as follows: first, we describe the benchmark policies and the implementation of our own policies. Then in Section~\ref{sec:synthetic}, we describe the setup and results of our synthetic experiments. Finally, in Section~\ref{sec:cosmic}, we test the performance of our fully adaptive algorithm on a publicly-available dataset of genetic mutations for cancer---COSMIC~\cite{tate2019cosmic, cosmic2019}.  

\paragraph{Algorithm Details}
In all algorithms, we start with a uniform prior, and update our prior distribution (over the hypotheses space) each time an observation is revealed. Unless otherwise mentioned, the algorithm terminates if the posterior probability of a hypothesis is above the threshold $1-\delta$.

\textbf{Random Baseline } 
At each step, an action was uniformly chosen from the set of all actions. 

\textbf{Partially Adaptive }
We implement the partially adaptive algorithm described in Section~\ref{sec:partial}, with the modifications that 1) the amount of boosting is now a built-in feature of the algorithm, and 2) breaking ties according to some heuristic.
%
We describe the modified algorithm in Appendix~\ref{appdx:partially adaptive algorithm}.

\textbf{Fully Adaptive }
We implement our algorithm described in Section \ref{sec:fa}, with the modifications that 1) the amount of boosting is considered as a tunable parameter, 2) a hypothesis is only considered to be ruled out when we are deciding which action to perform, 3) we do not boost if no action can further distinguish any hypotheses in the alive set, 4) we break ties according to some heuristic. In particular, Modification 1) addresses the issues that our fully adaptive algorithm in Section~\ref{sec:fa} over-boosts. Modification b) controls the error probability $\delta$ when we decrease the amount of boosting. Modification c) handles small $s$ without increasing the boosting factor. We formally describe this modified algorithm in Appendix~\ref{appdx:fully adaptive algorithm}.


\textbf{NJ Algorithms } \emph{NJ Adaptive} \cite{naghshvar2013active} is a two-phase algorithm that 
solves a relaxed version of our problem, where the objective is to minimize a weighted sum of the expected number of tests and the likelihood of identifying the wrong hypothesis, i.e., $\min \mathbb{E}(T)+L e$, where $T$ is the termination time, $L$ is the penalty for a wrong declaration, and $e$ is the probability of making that wrong declaration. 
The problem was formulated as a Markov decision process whose state space is the posterior distribution over the hypotheses.
In Phase 1, which lasts as long as the posterior probability of all hypotheses is below a carefully chosen threshold, the action is sampled according to a distribution that is selected to maximize the minimum expected KL divergence among all pairs of outcome variables. 
In Phase 2, when one of the hypotheses has posterior probability above the chosen threshold, $r$, the action is sampled according to a distribution selected to maximize the minimum expected KL divergence between the outcome of this hypothesis and 
the outcomes of all other hypotheses.
This threshold was optimized over in both synthetic and real-world experiments.
The algorithm stops if the posterior of a hypothesis is above the threshold $1-L^{-1}$. 
\emph{NJ Partially Adaptive} contains only the Phase 1 policy.

\begin{figure}[t]
    \centering
    \includegraphics[width = 0.31\textwidth]{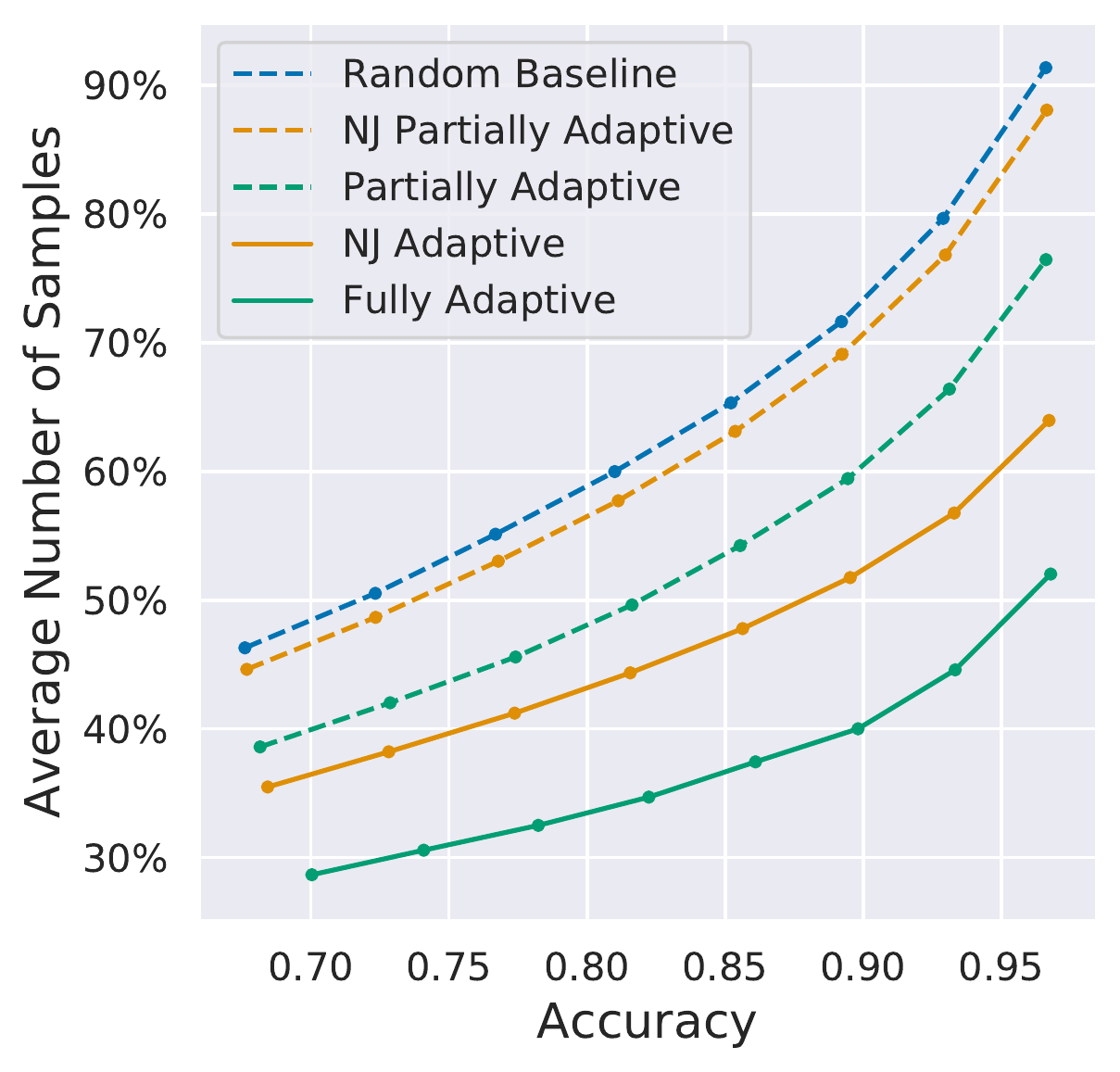}
    \includegraphics[width = 0.32\textwidth]{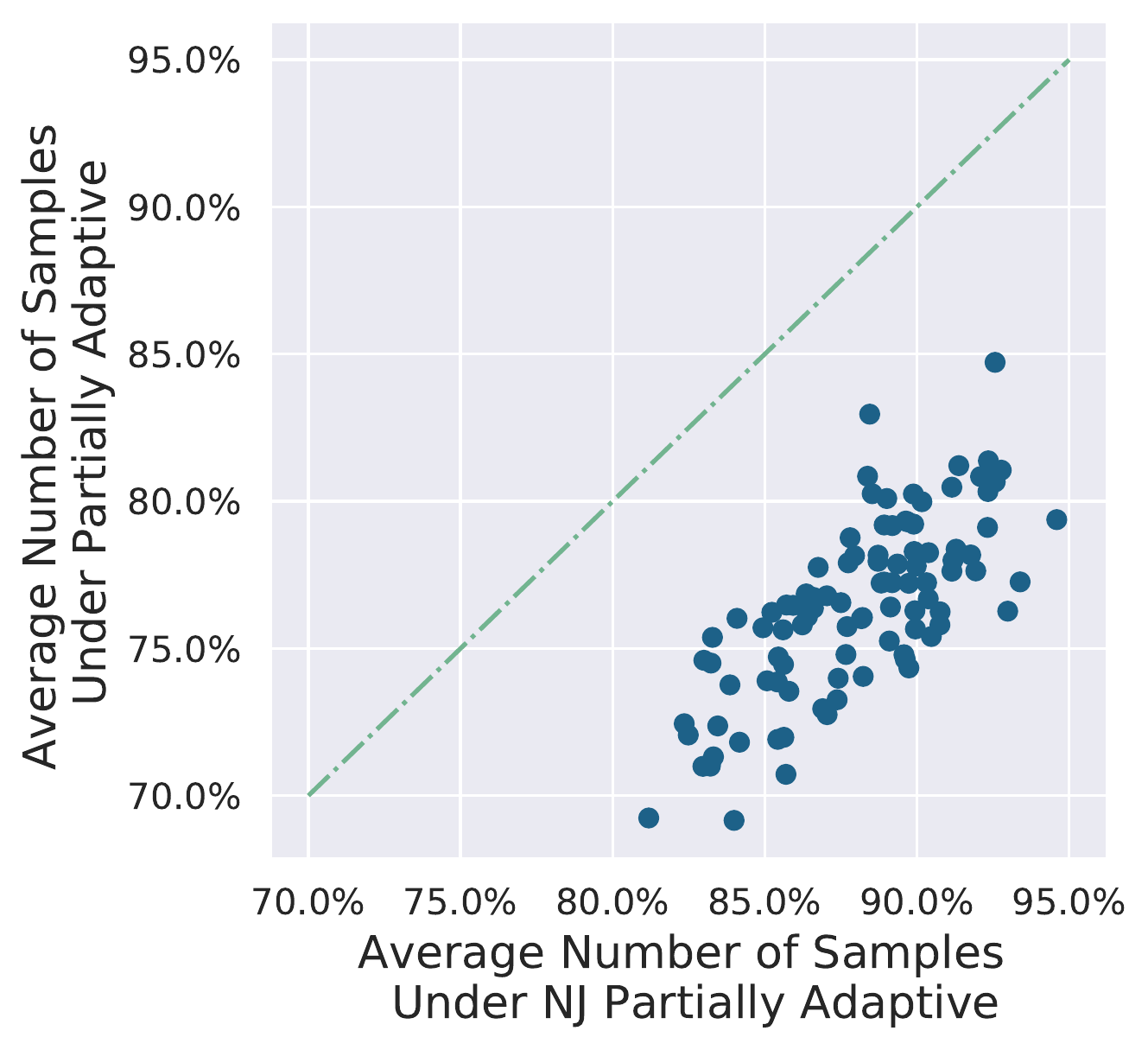}
    \includegraphics[width = 0.32\textwidth]{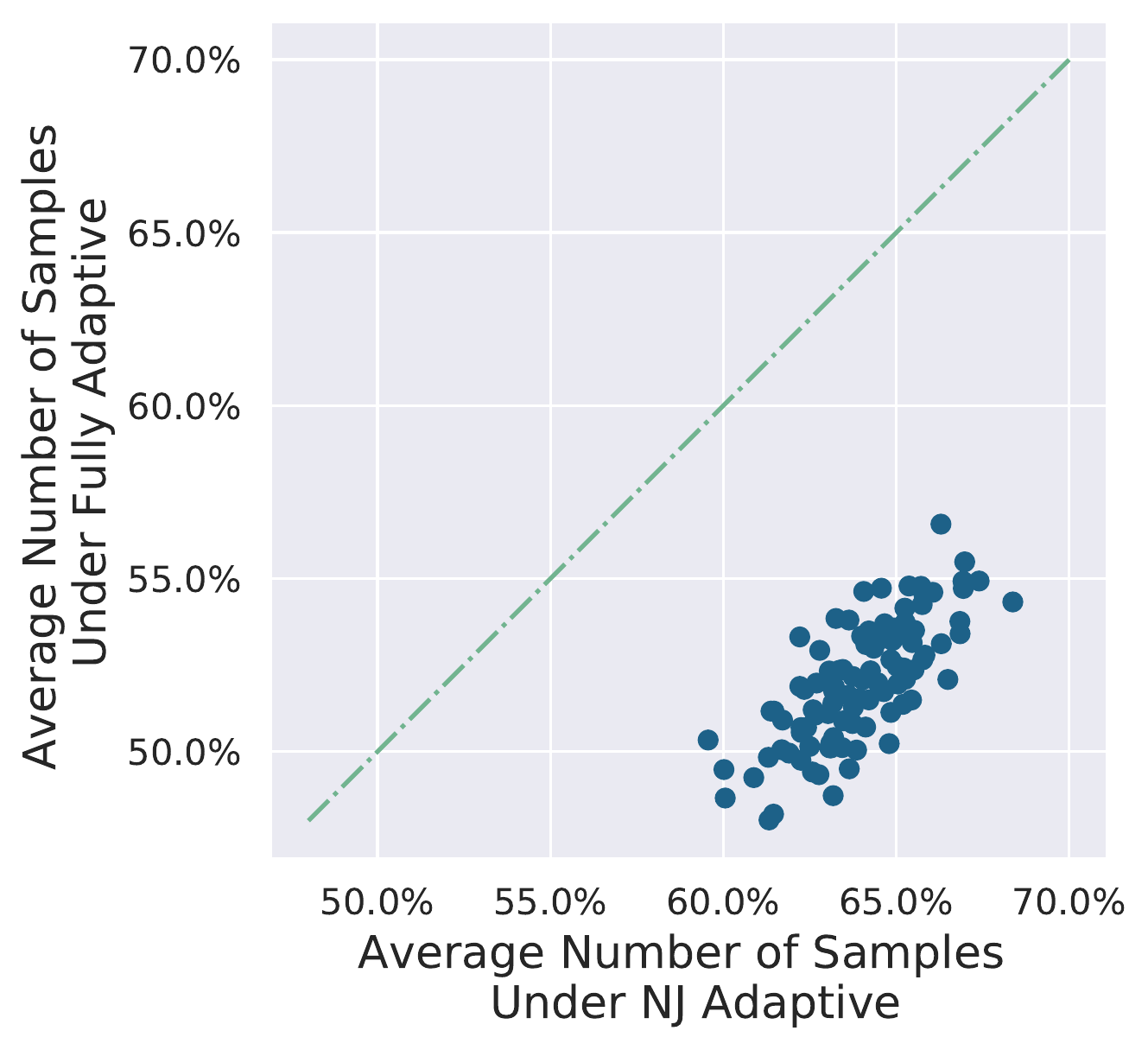}
    \vspace{-5pt}
    \caption{Comparison of our fully and partially adaptive algorithms with \emph{NJ Adaptive}, \emph{NJ Partially Adaptive} and \emph{Random Baseline} on synthetic data. The average number of samples is normalized with respect to the largest number of sample required in \emph{Random Baseline}.
    Left: each dot corresponds to the average performance of 100 randomly generated instances each averaged over 2,000 replications. Middle and Right: contains the same 100 instances in the left figure. Each dot  corresponds to one instance and each averaged over 2,000 replications.
    Middle and Right: 
    the average accuracies of those 100 instances in all algorithms equal to 0.97.
    %
    }
    \label{fig:partially_adaptive}
    \vspace{-10pt}
\end{figure}

\subsection{Synthetic Experiments}\label{sec:synthetic}
\paragraph{Parameter Generation and Setup  }
Figure~\ref{fig:partially_adaptive}
summarizes the results of our partially and fully adaptive experiments on synthetic data. 
Both figures
were generated with 100 instances: each with 25 hypotheses and 40 actions. The outcome of each action under each hypothesis is binary, i.e., the $D_{\mu(h,a)}$'s are the Bernoulli distributions, where 
$\mu(a, h)$ were \emph{uniformly} sampled from the [0,1] interval.
Each instance was then averaged over 2,000 replications, where a ``ground truth'' hypothesis was randomly drawn. The prior distribution, $\pi$, was initialized to be uniform for all runs. On the horizontal axis, the accuracies of both algorithms were averaged over these 100 instances, where the accuracy is calculated as the percentage of correctly identified hypotheses among the 2,000 replications. On the vertical axis, the number of samples used by the algorithm is first averaged over the 2,000 replications and then averaged over the 100 instances. 
\paragraph{Results  }
In Figure~\ref{fig:partially_adaptive} (left), we observe that 1) the performance of our fully adaptive algorithm dominates those of all other algorithms, 2) our partially adaptive algorithm outperforms all other partially adaptive algorithms, and 3) the performance of adaptive algorithms outperform those of partially adaptive algorithms. The threshold for entering Phase 2 policy in \emph{NJ Adaptive} was set to be 0.1. Indeed, we observe that \emph{NJ Adaptive} outperforms \emph{NJ Partially Adaptive}.
In Figure~\ref{fig:partially_adaptive} (middle), $\delta$ equals to 0.05 for both \emph{NJ Partially Adaptive} and \emph{Partially Adaptive}. We observe that our partially and fully adaptive algorithms outperform \emph{NJ Partially Adaptive} and \emph{NJ Adaptive}
instance-wise by large margins respectively in Figure~\ref{fig:partially_adaptive} middle and left.
\subsection{Real-World Experiments}\label{sec:cosmic}
\begin{figure}[t]
    \centering
    \includegraphics[width = 0.32\textwidth]{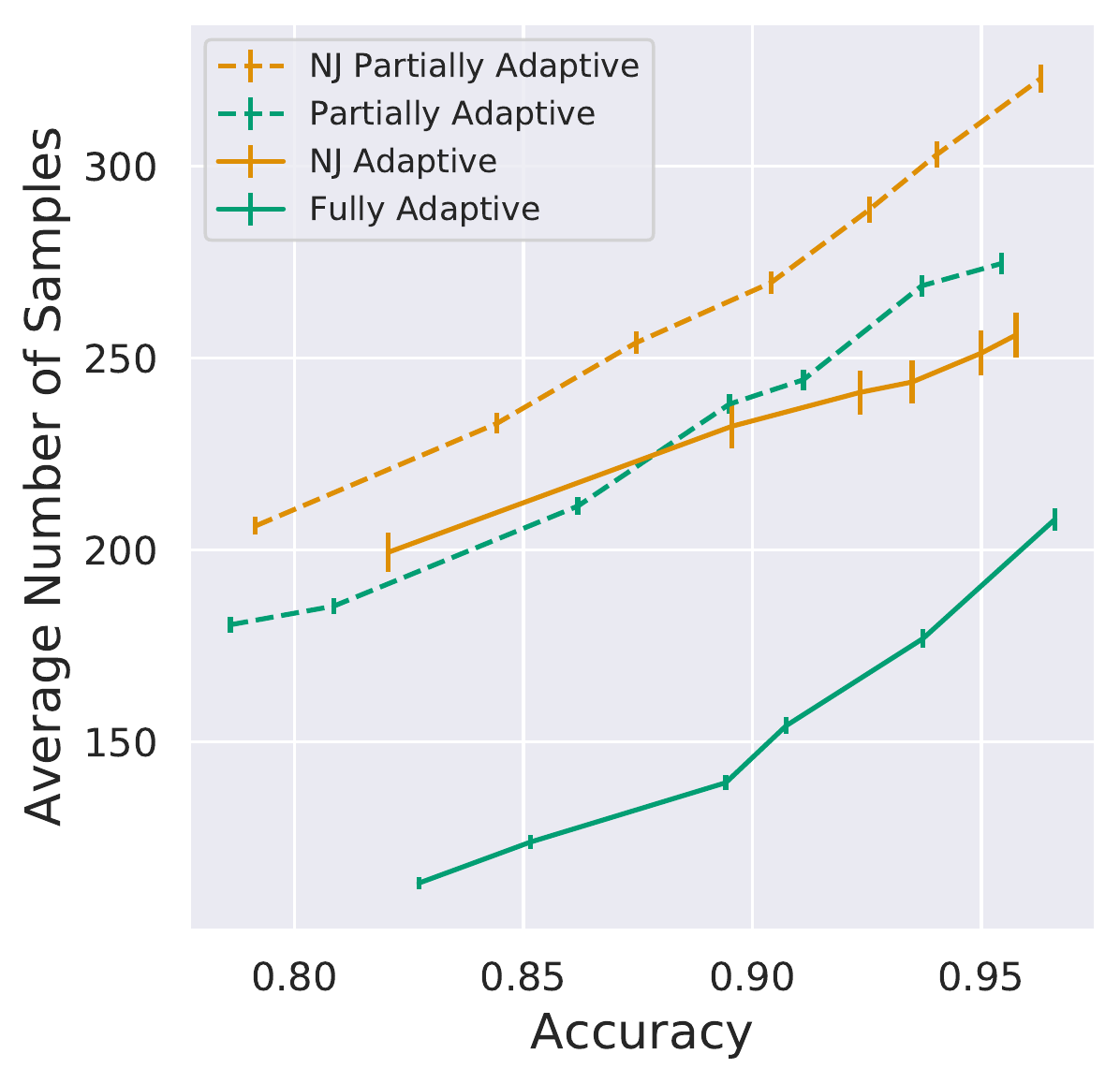}
    \includegraphics[width = 0.66\textwidth]{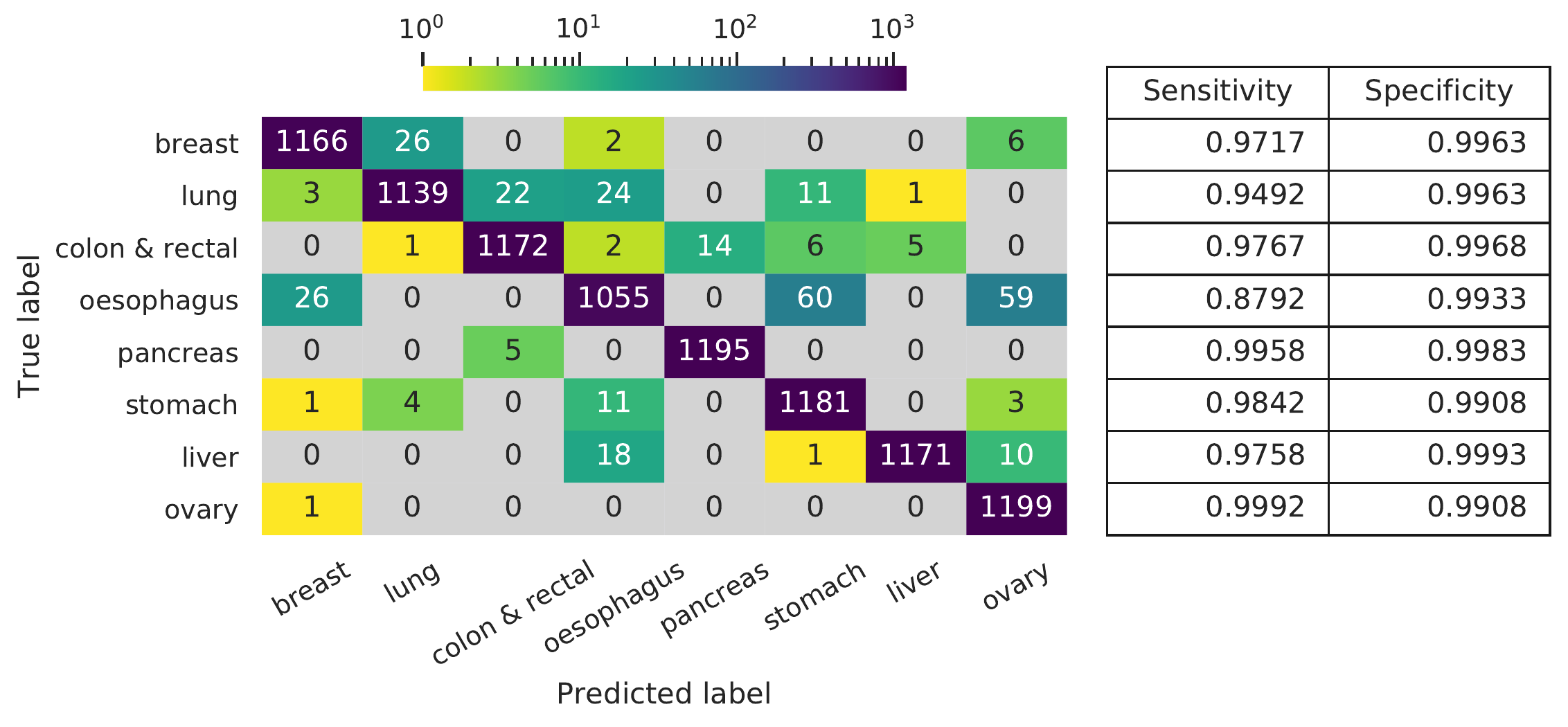}
    \vspace{-5pt}
    \caption{Comparison of our partially and fully adaptive algorithms with those of NJ's on real-world data---COSMIC. Left: each point is averaged over 9,600 replications. The error bars are the 95 percentage confidence intervals for the estimated means. Middle: the confusion matrix of \emph{Fully Adaptive} where the algorithm accuracy equals to 0.97, and each row sums up to 1,200. Left: the sensitivity and specificity our algorithm (middle) for each cancer type.}
\label{fig:cosmic}
    \vspace{-10pt}
\end{figure}

\paragraph{Problem Setup } Our real-world experiment is motivated by the design of DNA-based blood tests to detect cancer. In such a test, genetic mutations serve as potential signals for various cancer types, but DNA sequencing is, even today, expensive enough that the `amount' of DNA that can be sequenced in a single test is limited if the test is to remain cost-effective. For example, one of the most-recent versions of these tests \cite{cohen2018detection} involved sequencing just 4,500 {\em addresses} (from among 3 billion total addresses in the human genome), and other tests have had similar scale (e.g., \cite{razavi2017performance,chan2017analysis,phallen2017direct}).
Thus, one promising approach to the ultimate goal of a cost-effective test is adaptivity.

Our experiments are a close reproduction of the setup used by \cite{cohen2018detection} to identify their 4,500 addresses.
We use genetic mutation data from real cancer patients: the publicly-available \emph{catalogue of somatic mutations in cancer} (COSMIC)~\cite{tate2019cosmic, cosmic2019}, which includes the de-identified gene-screening panels for 1,350,015 patients.
We treated 8 different types of cancer (as indicated in \cite{cohen2018detection}) as the 8 hypotheses, and identified 1,875,408 potentially mutated genetic addresses. %
To extract the tests,
we grouped the the genetic addresses within an interval of 45 (see \cite{cohen2018detection} for the biochemical reasons behind this choice), resulting in 581,754 potential tests. We then removed duplicated tests (i.e., the tests that share the same outcome distribution for all 8 cancer types), 
resulting in 23,135 final tests that we consider in our experiments. {\rev Note that the duplicated tests can be removed here since they are exchangeable in our problem setting. However, the set of final tests might be different for a different set of ground-truth cancer types.}
From the data, we extracted a ``ground-truth'' table of mutation probabilities containing the likelihood of a  mutation 
in any of the 23,135 genetic address intervals being found in patients with any of the 8 cancer types. This served as the instance for our experiment.
The majority of the mutation probabilities in our instances was either zero or some small positive number. To calculate the KL divergence between these probabilities, we replace zero with the number $10^{-10}$ in our instance.

\paragraph{Results } Although in reality, all patients have different priors for having different cancers, in our experiments, we assume that
the truth hypothesis (cancer type) was drawn uniformly, and we initialize uniform priors for all algorithms.
Similar to Figure~\ref{fig:partially_adaptive}, in Figure~\ref{fig:cosmic} (left) we observe that 1) the performance of our algorithm dominates those to the rest algorithms, and 2) our partially adaptive algorithm outperforms \emph{NJ Partially Adaptive}. However, unlike Figure~\ref{fig:partially_adaptive}, we observe that \emph{NJ Adaptive} underperforms \emph{Partially Adaptive} when the accuracies are low on this instance. The threshold for entering Phase 2 policy, $r$, in \emph{NJ Adaptive} was set to be 0.3. Since Phase 1 policy is less efficient than Phase 2 policy, we observe that the performance of \emph{NJ Adaptive} is convex with respect to $r$---when $r$ is small, the algorithm is more likely to alternate between Phases 1 and 2 policies and when $r$ is large we spend more time in Phase 1 policy. As a result, we observe that variance of \emph{NJ Adaptive} is relatively high when compared with those of other algorithms.
Note that due to the nature of the sparsity of our instance, the performance of the random baseline was very poor when compared with these of \emph{NJ Adaptive} and \emph{Fully Adaptive} and thus was excluded.
Figure~\ref{fig:cosmic} (middle) is the confusion matrix corresponding to our fully adaptive algorithm where the algorithm accuracy equals to 0.97, and Figure~\ref{fig:cosmic} (right) corresponds to the sensitivity and specificity of our algorithm for each cancer type (in the same ordering as) in the middle figure. 
\vspace{-3pt}
\section{Conclusions}\label{sec:lim}
\vspace{-4pt}
In this work we provided the first approximation guarantees for the ASHT problem and demonstrated the efficiency of the proposed \alg s through numerical experiments on genetic mutation data.
Under the current framework, it is challenging to improve the $s^{-1}$ dependence on $s$ in the \apxn\  factors, since we have to boost each action for $s^{-1}$ times to apply the concentration bounds. 
However, in our numerical experiments, by reducing the number of times for boosting, we achieved better performance when compared with existing heuristics.
\newpage
\begin{ack}
We would like thank the anonymous reviewers for their careful reviews, and we declare no conflict of interests. 
\end{ack}
\bibliographystyle{abbrv}
\bibliography{references}
\newpage
\appendix
\section{Prerequisite: Subgaussian Random Variables}\label{apdx:prereq}
We will consider the commonly used \sg\ \distr s (\cite{vershynin2018high}). 
Loosely speaking, a \rv\ is \sg\ if its tail vanishes at a rate faster than some Gaussian distributions. 
\bdefn[\Sg\ norm]
Let $X$ be a \rv, its {\it \sg\ norm} is defined as
$\|X\|_{\psi_2}:= \inf \{t: \ho{E} [e^{X^2/t^2}] \leq 2\}$.
Moreover, $X$ is called {\it \sg} if $\|X\|_{\psi_2}<\infty$.
\edefn
Many commonly used \distr s \sat\ this \assu, e.g., Bernoulli, uniform, and  Gaussian \distr s. 
We introduce a standard concentration bound for \sg\ \rv s.
\begin{theorem}[Hoeffding Inequality~\cite{vershynin2018high}]\label{thm:hoeffding}
Let $X_1,...,X_n$ be \indep\ \sg\ \rv s.
Then for any $\eta >0$, it holds that 
\[\ho{P}\left[\left|\sum_{i=1}^n X_i - \sum_{i=1}^n \ho{E}(X_i)\right| \geq \eta \right]\leq 2\exp\left(-\frac{2 \eta^2}{\sum_{i=1}^n \|X_i\|_{\psi_2}^2}\right).\]
\end{theorem}
To show the correctness of our \alg, we need to consider the \emph{log-likelihood ratio} (LLR), formally defined as follows:
\bdefn
For any $a\in A$ and $h,g\in H$, define $Z(h,g;a)=\log \frac{\ho{P}_{h,a}(\xi)}{\ho{P}_{g,a}(\xi)}$ where $\xi\sim D_{\mu(h,a)}$. 
\edefn
We will assume that the subgaussian norm of the LLR between two \hypos\ is not too large when compared to the difference of their \pmt s, as formalized below:
\bdefn
Let $\rho>0$ be the minimal number s.t. for any pair of distinct  hypotheses $h,g\in H$ and action $a\in A$, it holds that $\|Z(h,g;a)\|_{\psi_2}\leq \rho \cdot |\mu(g,a) - \mu(h,a)|$.
\edefn
We will present an error analysis for general $\rho$. 
Prior to that, we first point out that 
many common \distr s \sat\ $\rho=O(1)$.

\noindent{\bf Examples.}
It is \strfwd\ to verify that $\rho=O(1)$ for the following common \distr s:
\bitem 
\item Bernoulli \distr s: $D_\theta = Ber(\theta)$ where $\theta \in [\theta_{min},\theta_{max}]$
for constants $\theta_{min},\theta_{\max}\in (0,1)$, and 
\item Gaussian \distr s: $D_\theta = N(\theta,1)$ where $\theta \in [\theta_{min},\theta_{max}]$
for constants $\theta_{min}< \theta_{\max}$.
\eitem 
Take Bernoulli \distr\ as an example. 
Fix any \hypos\ $h,g\in H$ and action $a\in A$, write $\Delta = \mu(h,a) - \mu(g,a)$. 
Then, $Z=Z(h,g;a)$ can be rewritten as
\[Z =
\begin{cases}
\log (1+\frac{\Delta}{\mu(g,a)}), \quad &\text{ w.p. } \mu(h,a),\\
\log (1-\frac{\Delta}{1-\mu(g,a)}), \quad &\text{ w.p. } 1-\mu(h,a).
\end{cases}\]
Since $0<\theta_{min}\leq \mu(g,a) \leq \theta_{\max}<1$, we have $|Z| \leq C|\Delta|$  almost surely 
where $C=2\max\{(1-\theta_{max})^{-1},\theta_{min}^{-1}\}$.
Moreover, it is known that (see \cite{vershynin2018high}) any \sg\ \rv\ $Z$ \sats\ $\|Z\|_{\psi_2} \leq \frac 1{\ln 2} \|Z\|_{\infty}$,
so \ift\ 
\[\|Z\|_{\psi_2} \leq \frac 1{\ln 2} \|Z\|_{\infty} \leq \frac {C\Delta}{\ln 2}=O(\Delta).\]
Thus $\rho = O(1)$.



\section{Proof of Proposition~\ref{prop:pa}}\label{sec:pf_pa}
\subsection{Error Analysis}\label{sec:error_analysis}
We first prove that at each timestamp $\tau(h)$, with high \prb\ our algorithm terminates and returns $h$.
\begin{lemma}\label{lem:correctness}
Let $B>0$.
If $h\in H$ is the true \hypo, then w.p. $1-e^{-\Omg(\rho^{-2} \alpha B)}$, it holds $\log \Lambda(h,g;\tau(h))\geq \frac{1}{2}\alpha B$ for all $g\neq h$.
\end{lemma}

\begin{proof}
Let $\tilde\sigma=(a_1,a_2,...)$ be the \xulie\ {\it after} the boosting step, so $a_1=...=a_\alpha, a_{\alpha +1} = ...=a_{2\alpha}$, so on so forth.
Write $Z_i=Z(h,g;a_i)$, then for any $t\geq 1$, it holds $\log \Lambda(h,g;t) = \sum_{i=1}^t Z_i$.
By the definition of cover time, $\sum_{i=1}^{\tau(h)} d(h,g;a_i)
= \sum_{i=1}^{\tau(h)} \ho{E}(Z_i) \geq \alpha B.$
Thus, 
\begin{align}
\ho{P}_h \left[\log\Lambda\big(h,g;\tau(h)\big)< \frac{1}{2}\alpha B\right]
&=\ho{P}_h\left[\sum_{i=1}^{\tau(h)} Z_i < \frac{1}{2}\alpha B\right] \notag\\
&\leq \ho{P}_h\left[\left|\sum_{i=1}^{\tau(h)} Z_i - \sum_{i=1}^{\tau(h)} \ho{E}(Z_i)\right| > \frac{1}{2}\sum_{i=1}^{\tau(h)} \ho{E}(Z_i)\right].\label{eqn:oct13}
\end{align}
By Theorem~\ref{thm:hoeffding},
\beqn\label{eqn:oct10}
\text{Equation } (\ref{eqn:oct13})\leq \exp \left(-\Omg\left(\frac{(\alpha B)^2}{\sum_{i=1}^{\tau(h)} \|Z_i\|_{\psi_2}^2}\right)\right).
\eeqn
We next show that $\sum_{i=1}^{\tau(h)} \|Z_i\|_{\psi_2}^2 \leq O(\rho^{2} \alpha B)$.
Write $\Delta_i = \mu(h,a_i) - \mu(g,a_i)$, then by Assumption~\ref{assu1}, $\Delta_i^2 \leq C_2 \cdot d(h,g;a_i)$. 
Note that $\|Z_i\|_{\psi_2}\leq \rho \Delta_i$, so \ift\
\begin{align}\label{eqn:may27}
\sum_{i=1}^{\tau(h)} \|Z_i\|_{\psi_2}^2 \leq \rho^2 \sum_{i=1}^{\tau(h)} \Delta_i^2 \leq C_2 \rho^2 \sum_{i=1}^{\tau(h)} d(h,g;a_i). 
\end{align}
Recall that $\sigma$ is the \xulie\ {\it before} boosting. 
Write $t=CT(f_h^B,\sigma)$ for simplicity.
By definition of cover time, 
\begin{align*}
\sum_{i=1}^{\alpha t} d(h,g;a_i)\geq \alpha B \geq \sum_{i=1}^{\alpha(t-1)} d(h,g;a_i).
\end{align*}
Note that $\tau(h)=\alpha t$, so \[\sum_{i=1}^{\alpha t} d(h,g;a_i) \leq 2\sum_{i=1}^{\alpha (t-1)} d(h,g;a_i) \leq 2\alpha B.\]
Combining the above with Equation (\ref{eqn:may27}), we have
\[\sum_i \|Z_i\|_{\psi_2}^2 \leq 2C_2 \rho^2 \alpha B.\]
Substituting into Equation (\ref{eqn:oct10}), we obtain 
\[\ho{P}_h \left[\log\Lambda\big(h,g;\tau(h)\big)< \frac{1}{2}\alpha B\right] \leq e^{-\Omg(\rho^{-2} \alpha B)}.\] 
The proof completes by applying the union bound over all $g\in H\bs \{h\}$.
\end{proof}

By a similar approach we may also show that it is unlikely that the \alg\ terminates at a wrong time stamp before scanning the correct one.
\begin{lemma}\label{lem:correctness2}
Let $B>0$. 
If $h\in H$ is the true \hypo, then for any $g\neq h$, it holds that $\log \Lambda(g,h;\tau(g))< \frac{1}{2}\alpha B$ with probability $1-e^{-\Omg(\rho^{-2}\alpha B)}$.
\end{lemma}

We are able to bound the error of the RnB algorithm by combining Lemma~\ref{lem:correctness} and Lemma~\ref{lem:correctness2}.
\begin{proposition}
For any true \hypo\ $h\in H$, algorithm $RnB(B,\alpha)$ returns $h$ with \prb\ at least $1-|H|e^{-\Omg(\rho^{-2}\alpha B)}$.
In \parti, if the outcome \distr\ $D_\mu$ is $Ber(\mu)$, then $\rho=O(1)$ and the above \prb\ becomes $1-|H|e^{-\Omg(\alpha B)}$.
\end{proposition}

\subsection{Cost Analysis}
Recall that in Section~\ref{sec:partial}, only Step (C) remains to be shown, which we formally state below.
\begin{proposition}\label{step_C}
Let $(\sigma,T)$ be a $\delta$-PAC-error \pa\ \alg. 
For any $B\leq \log\delta^{-1}$ and $h\in H$, it holds $\ho{E}_h [T] \geq \Omg\big(s \cdot \mathrm{CT}(f^B_h,\sigma) \big ).$
\end{proposition}

We fix an \arb\ $h\in H$ and write $CT_h:=CT(f^B_h, \sigma)$, where we recall that $\sigma$ is the \xulie\ of actions before boosting (do not confuse with $\tilde \sigma$).
To relate the stopping time $T$ (under $h$) to the cover time of the submodular function for $h$ in $\sigma$, we introduce a linear program. 
We will show that for suitable choice of $d$, we have  
\bitem 
\item $LP^*(d, CT_h -1)\leq \ho{E}_h T$, and 
\item $LP^*(d, CT_h -1)\geq \Omg(s\cdot CT_h)$.
\eitem 
Hence proving Step (C) in the high-level proof sketched in Section~\ref{sec:partial}.

We now specify our choice of $d$.
For any $d_1,...,d_N\in\real_+$, write $d^t := \sum_{i=1}^t d_i$ for any $t$ and consider
\beqn
LP(d, t):\quad \min_z &\sum_{i=1}^N i \cdot z_i\\
s.t.\  & \sum_{i=1}^N d^i z_i \geq d^t,\\
& \sum_{i=1}^N z_i = 1,\\
&\quad z\geq 0.\notag
\eeqn

We will consider the following choice of $d_i$'s. 
\Sps\ $(\sigma, T)$ has $\delta$-PAC-error where $\delta \in (0,1/4]$. 
For any pair of \hypos\ $h,g$ and any set of actions $S$, define 
\[K_{h,g}^B (S) = \min\left\{1,
B^{-1} \sum_{a\in S} d(h,g;a)\right\}.\]
Hence, \[f_h^B(S) = \frac 1{|H|-1} \sum_{g\in H\bs \{h\}} K_{h,g}^B (S).\]
Fix any $B \leq \log \delta^{-1}$ and let $g$ be the last \hypo\ \sep ed from $h$, i.e., \[g:=\arg \max_{h'\in H\bs \{h\}}\left\{\mathrm{CT}(K_{h,h'}^B,\sigma)\right\}.\] 
Then by the definition of cover time, we have $\mathrm{CT}_h =\mathrm{CT}(f^B_h, \sigma) = \mathrm{CT}(K_{hg}^B,\sigma).$
Without loss of generality,\footnote{If there is some action $a$ with $d(h,g;a)=0$, then we simply remove it. This will not change the argument.} we assume that all actions $a$ satisfy $\mu(h,a)=\mu(g,a)$ in $\tilde \sigma=(a_1,..,a_N)$.
We choose the LP parameters to be $d_i = d(h,g,a_i)$ for $i\in [N]$.

{\bf Outline.} We will first show that the LP optimum is upper bounded by the expected termination time $T$ (Proposition~\ref{lem:feas}), and then lower bound it in terms of $\mathrm{CT}_h$ (Proposition~\ref{lem:lp}).

\begin{proposition}\label{lem:feas}
\Sps\ $(\sigma, T)$ has $\delta$-PAC-error for some $0<\delta\leq \frac 14$. 
Let $z_i=\ho{P}_h[T=i]$ for $i\in [N]$, then $z=(z_1,...,z_N)$ is \feas\ to $LP(d, \mathrm{CT}_h-1)$.
\end{proposition}

Note that $\ho{E}_h (T)$ is simply the objective value of $z$, thus Proposition~\ref{lem:feas} immediately implies: 
\bcoro
$\ho{E}_h (T) \geq LP^*(d, \mathrm{CT}_h-1)$.
\ecoro

We next lower bound the expected log-likelihood when the \alg\ stops.  
\begin{lemma}\cite{Nowak09}
\label{lem:expected_LLR}
Let $\ho{A}$ be any algorithm (not necessarily partially \adap) for the ASHT problem. 
Let $h,g\in H$ be any pair of distinct \hypos\ and $O$ be the random output of $\ho{A}$. 
Define the error probabilities $P_{hh} =\ho{P}_h (O=h)$ and  $P_{hg}=\ho{P}_h(O=g)$. 
Let $\Lambda$ be the likelihood ratio between $h$ and $g$ when $\ho{A}$ terminates. Then,
\[\ho{E}_h(\log \Lambda) \geq P_{hh}\log\frac{P_{hh}}{P_{hg}}
+ (1-P_{hh})\log\frac{1-P_{hh}}{1-P_{hg}}.\]
\end{lemma}

\begin{proof}
Let $\mathcal{E}$ be the event that the output is $h$. Then by Jensen's \ineq, we have
\beqn\label{eqn:sep24}
\ho{E}_h (\log \Lambda_T|\mathcal{E}) 
\geq -\log \ho{E}_h \left(\Lambda^{-1}|\mathcal{E}\right)
=-\log \frac{\ho{E}_h \left(\mathbbm{1}(\mathcal{E})\cdot \Lambda^{-1}\right)}{\ho{P}_h(\mathcal{E})}. \eeqn

Recall that an algorithm can be viewed as a decision tree in the following way: each internal node is labeled with an action, and each edge below it corresponds to a possible outcome; each leaf corresponds to termination and is labeled with a \hypo\  corresponding to the output.
Write $\sum_x$ as the summation over all leaves and let $p_h(x)$ (resp. $p_g(x)$) be the \prb\ that the \alg\ terminates in leaf $x$ under $h$ (resp. $g$), then, 
\beqn
\ho{E}_h\left(\mathbbm{1}(\mathcal{E})\cdot \Lambda^{-1}\right)
&=\sum_{x}\mathbbm{1}(x\in \mathcal{E}) \cdot \Lambda^{-1}(x)\cdot p_h(x)\\
&= \sum_{x}\mathbbm{1}(x\in \mathcal{E}) \cdot \frac{p_g(x)}{p_h(x)}p_h(x)\\
&= \sum_{x} \mathbbm{1}(x\in \mathcal{E})\cdot p_h(x) \\
&= \ho{E}_h (\mathbbm{1}(x\in \mathcal{E})) = P_{hg}.
\notag\eeqn
Combining the above with Equation (\ref{eqn:sep24}), we obtain
\[\ho{E}_h (\log \Lambda|\mathcal{E}) 
\geq \log\frac{P_{hh}}{P_{hg}}.\] 
Similarly, we have $\ho{E}_h \left(\log \Lambda|\mathcal{\bar E}\right)
\geq \log\frac{1-P_{hh}}{1-P_{hg}}$, where $\mathcal{\bar E}$ is the event that the output is not $h$. The proof follows immediately by combining these two \ineqs.
\end{proof}

To show Proposition~\ref{lem:feas}, we need a standard concept---stopping time.
\bdefn[Stopping time~\cite{mitzenmacher2017probability}]
Let $\{X_i\}$ be a \xulie\ of \rv s and $T$ be an integer-valued \rv. If for any integer $t$, the event $\{T=t\}$ is \indep\ with $X_{t+1}, X_{t+2},...$, then $T$ is called a {\bf{stopping time }} for $X_i$'s.
\edefn
\begin{lemma}[Wald's Identity]\label{lem:wald}
Let $\{X_i\}_{i\in \mathbb{N}}$ be \indep\ \rv s with means $\{\mu_i\}_{i\in \mathbb{N}}$, and let $T$ be a stopping time w.r.t. $X_i$'s. Then, 
$\ho{E}\left(\sum_{i=1}^T X_i\right) = \ho{E}\left(\sum_{i=1}^T \mu_i\right).$
\end{lemma}

{\bf Proof of Proposition~\ref{lem:feas}.}
One may verify that the lower bound in Lemma~\ref{lem:expected_LLR}
is increasing w.r.t $P_{hh}$ and decreasing w.r.t $P_{hg}$. 
Therefore, since $\ho{A}$ has $\delta$-PAC-error, by Lemma~\ref{lem:expected_LLR} it holds that
\[\ho{E}_h \left(\log \Lambda(h,g;T)\right) \geq 
(1-\delta)\log\frac{1-\delta}{\delta}+\delta\log\frac{\delta}{1-\delta}\geq\frac{1}{2}\log\frac{1}{\delta} \geq B\geq d^{\mathrm{CT}_h -1}.\]

By Lemma~\ref{lem:wald}, 
\[ \sum_{i=1}^N d^i z_i = \sum_{i=1}^N d^i \cdot \ho{P}_h(T=i) = \ho{E}_h \left(\log \Lambda(h,g;T)\right). \]
The proof follows by combining the above. \qed

So far we have upper bounded $LP^*(d,\mathrm{CT}_h-1)$ using $\ho{E}_h (T)$. 
To complete the proof, we next lower bound $LP^*(d,\mathrm{CT}_h-1)$ by $\Omg(s\cdot \mathrm{CT}_h)$.
\begin{lemma}\label{lem:two_pt_distr}
$LP^* = \min_{i\leq t<j} LP^*_{ij}$ where $LP^*_{ij} = i + (j-i)\frac{d^t - d^i}{d^j - d^i}.$
\end{lemma} 

\begin{proof}
Observe that for any optimal \sln, the \ineq\ constraint must be tight. 
By linear algebra, we deduce that any basic feasible solution has support size two. 

Consider the solutions whose only nonzero entries are $i,j$. Then, $LP(d,t)$ becomes
\beqn
LP_{ij}(d, t):\quad \min_{z_i,z_j} \quad & iz_i + jz_j\\
s.t.\  & d^i z_i + d^j z_j = d^t,\\
& z_i+z_j = 1,\\
&\quad z\geq 0.\notag
\eeqn
Note that since $d^i<d^j$, $LP_{i,j}(d,t)$ admits exactly one \feas\ \sln, whose \obj\ value can easily be verified to be
$LP^*_{ij} := i + (j-i)\frac{d^t - d^i}{d^j - d^i}$.
\end{proof}

Now we are ready to lower bound the LP optimum.
\begin{proposition}\label{lem:lp}
For any $d=(d_1,...,d_N)\in \real^N$ and $t\in \mathbb{N}$, it holds that 
$LP^*(d, t) \geq t\cdot \min\{d_i\}_{i\in [N].}$
\end{proposition}

\begin{proof}
By Lemma~\ref{lem:two_pt_distr}, it suffices to show that $LP^*_{ij} \geq d^t$ for any $i\leq t<j$. 
Since $d^k<k$ for any integer $k$, 
\[(j-d^t)(d^t-d^i) \geq (d^j-d^t)(d^t-i).\]
Rearranging, the above becomes
\[i(d^j-d^i) + (j-i)(d^t -d^i)\geq d^t (d^j-d^i),\]
i.e., 
\[i+(j-i)\frac{d^t - d^i}{d_j - d^i}\geq d^t.\]
Note that the LHS is exactly $LP_{ij}^*$, thus $LP^*(d, t)\geq d^t\geq t \cdot \min\{d_i\}_{i\in [N]}$ for any $t\in \mathbb{N}$.
\end{proof}
It immediately follows that $LP^*(d, t)\geq st$, completing the proof of Proposition~\ref{prop:pa}.



\section{Proof of Proposition~\ref{prop:convert}}\label{proof:prop2}
We first formally define a decision tree, not only for mathematical rigor but more importantly, for the sake of introducing a novel variant of ODT. 
Recall that $\Omg$ is the space of the test outcomes, which we assume to be discrete for simplicity.
\bdefn[Decision Tree]
A \dec\ tree is a rooted tree, each of whose interior (i.e., non-leaf) node  $v$ is associated with a {\it state} $(A_v, T_v)$, where $T_v$ is a test and $A_v\subseteq H$. 
Each interior node has $|\Omg|$ children, each of whose edge to $v$ is labeled with some outcome. 
Moreover, for any interior node $v$, the set of {\it alive} \hypos\ $A_v$ is the set of \hypos\ consistent with the outcomes on the edges of the path from the root to $v$.
A node $\ell$ is a {\it leaf} if $|A_\ell|=1$. 
The \dec\ tree terminates and outputs the only alive \hypo\ when it reaches a leaf.
\edefn

To relate $OPT^{FA}_\delta$ to the optimum of a suitable ODT \ins, we introduce a novel variant of ODT.
As opposed to the \ord\ ODT where the output needs to be correct with \prb\ $1$, in the following variant, we consider \dec\ trees which may {\it err} sometimes:
\bdefn[Incomplete Decision Tree]
An {\it incomplete \dec\ tree} is a decision tree 
whose leaves $\ell$'s are associated with {\it state}s $(A_\ell, p_\ell)$'s, 
where $A_\ell$ represents the subset of \hypos\ consistent with all outcomes so far,  and $p_\ell$ is a \distr\ over $A_\ell$.  
A \hypo\ is randomly drawn from $p_\ell$ and is returned as the identified \hypo\ (possibly wrong).
\edefn

Now we already to introduce {\it chance-constrained ODT problem} (CC-ODT).
Given an error budget $\delta>0$, we aim to find the minimal cost \dec\ tree whose error is within $\delta$.
There are two natural ways to interpret ``error'', which both will be considered in Appendices~\ref{proof:prop2} and \ref{append:total_error}.
In the first one, we require the error \prb\ under {\it any} \hypo\ to be lower than the given error budget. 
In the other one, we only require the {\it expected} error \prb\ over all \hypos\ to be within the budget.
Intuitively, the second version allows for more flexibility since the errors under different \hypos\ may differ significantly, rendering the analysis more challenging since we do not know how the error budget is allocated to each \hypo.
We formalize these two versions below.
Let $O$ be the random outcome returned by the tree.

\textbf{CC-ODT with PAC-Error.}
An incomplete \dec\ tree is {\it{$\delta$-PAC-Valid}} if, for any true \hypo\ $h$, it returns $h$ with \prb\ at least $1-\delta$, formally, 
\begin{align*}
    \ho{P}_h (O\neq h) \leq \delta, \quad \forall h\in H.
\end{align*}

\textbf{CC-ODT with Total-Error.}
An incomplete \dec\ tree is {\it {$\delta$-Total-Valid}} if, for the total error \prb\ is at most $\delta$, formally,
$$\sum_{h\in H} \pi(h) \cdot \ho{P}_h(O\neq h )\leq \delta,$$
where $\pi$ is the prior \distr.
The goal in both versions is to find an  incomplete \dec\ tree with minimal expected cost, subject to the \corres\ error \constr.

For the proof of Proposition~\ref{prop:convert}, consider the PAC-error version of CC-ODT.
It turns out that this version of CC-ODT is indeed quite trivial (unlike the total-error version):
below we show that under PAC-error, CC-ODT is almost equivalent to the ordinary ODT problem.

\begin{lemma}\label{lem:odt_delta}
\Sps\ $\delta \in (0,\frac 12)$, 
and $\ho{T}$ is a $\delta$-PAC-valid \dec\ tree. 
Then, $\ho{T}$ must also be $0$-valid.
\end{lemma}
\begin{proof}
It suffices to show that there is no incomplete node in $\ho{T}$.
For the sake of  contradiction, assume $\ho{T}$ has an incomplete node $\ell$ with state $(A_\ell, p_\ell)$. 
By the definition of incomplete node, $|A_\ell|\geq 2$, so there is an $h\in A_\ell$ with $p_\ell(h)\leq \frac 12$.
Now \sps\ $h$ is the true \hypo. 
Since each \hypo\ traces a unique path in any \dec\ tree, regardless of whether or not it is incomplete, $h$ will reach node $\ell$ with \prb\ $1$.
Then at $\ell$, the \dec\ tree returns $h$ with \prb\ 
$p_\ell(h) = 1- \sum_{g\in A_\ell: g\neq h} p_\ell(g)\leq \frac 12,$ 
and hence $\ho{P}_h[O\neq h] \geq \frac 12$, reaching a contradiction.
\end{proof}
For the reader's convenience, we  recall that an ASHT \ins\ $\cal{I}$ is associated with an ODT \ins\ $\mathcal{I}_{ODT}$, defined as follows.
Each action corresponds to a test $T_a: H\rar \Omg_a$ with $T_a(h)=\mu(h,a)$, where $\Omg_a = \{\mu(h,a): h\in H\}$, and the cost $T_a$ is $c(a)= \lceil s(a)^{-1} \log (|H|/\delta)\rceil$.
Denote $ODT_\delta^*$ the minimal cost of any $\delta$-PAC-valid decision tree for $\mathcal{I}_{ODT}$. 
Then we immediately obtain the following from the Lemma~\ref{lem:odt_delta}.
\begin{corollary}
If $\delta \in (0,\frac 12)$, then $ODT^*_0 = ODT^*_\delta$.
\end{corollary}

Now we are able to complete the proof of the main proposition.

{\bf Proof of Proposition~\ref{prop:convert}.}
Given a $\delta$-PAC-error algorithm $\ho{A}$, we show how to construct a $\delta$-PAC-valid \dec\ tree $\ho{T}$ as follows. 
View $\ho{A}$ as a \dec\ tree (discretize the outcome space if it is continuous).
Replace each action $a$ in $\ho{A}$ with the test $T_{a}$.
Note that the cost of $T_a$ is $s(a)^{-1} \log (|H|/\delta)\leq s^{-1} \log (|H|/\delta)$. 
Therefore by Lemma~\ref{lem:odt_delta}, 
\begin{align*}
ODT^*_0 &= ODT_\delta^*\leq c(\ho{T}) \leq s^{-1}\log \frac{|H|}{\delta}\cdot OPT^{FA}_\delta. &\qed
\end{align*}

\section{Total Error Version}\label{append:total_error}
In the last section we defined the total-error version of the CC-ODT problem. 
The total error version of the ASHT problem can be defined analogously, so we do not repeat it here.
We say an algorithm is said to be \textbf{$\delta$-total-error} if the total probability (averaged with respect to the prior $\pi$) of erroneously identified a wrong hypothesis is at most $\delta$.
%
%
The following is our main result for the total-error version.
\thmODT*
In \parti, if $\delta\leq O(|H|^{-1/2})$, then the above is polylog-\apxn\ for fixed $s$.

We will first prove Theorem~\ref{thm:odt} for the fully adaptive version, and then show how the same proof works for the partially adaptive version. Unlike the PAC-error version where CC-ODT is almost equivalent to ODT, in the total-error version their optima can differ by a $\Omg(|H|)$ factor.
We construct a sequence of ODT instances $\mathcal{I}_n$, where $n\in \mathcal{Z}^+$, with $ODT_\delta^*(\mathcal{I}_n)/ODT_0^*(\mathcal{I}_n)=O(\frac 1n)$.
\Sps\ there are $n+2$ \hypos\ $h_1,...,h_n$ and $g,h$, with $\pi(g)=\pi(h) = 0.49$ and $\pi(h_i)= \frac{1}{50 n}$ for $i=1,...,50$. 
Each (binary) test partitions $[n+2]$ into a singleton and its complement. 
Consider error budget $\delta=\frac 14$, then for each $n$ we have $ODT_\delta^*(\mathcal{I}_n)=1$. 
In fact, we may simply perform a test to \sep e $g$ and $h$, and then return the one (out of $g$ and $h$) that is consistent with the outcome.
The total error of this \alg\ is $1/50<\delta$. 
\OTOH, $ODT_0^*(\mathcal{I}_n)=n+1$.


However, for uniform prior, this gap is bounded:
\bprop\label{lem:chance_constr}
\Sps\ the prior $\pi$ is \unif. 
Then, for any $\delta\in (0,\frac{1}{4})$, it holds 
\[ODT^*_0\leq \big(1+O(|H| 
\delta^2)\big )\cdot ODT_\delta^*.\]
\eprop
To show the above, we need the following building block.
\begin{lemma}\label{lem:inc_nodes}
\Sps\ the prior $\pi$ is \unif. 
Then, for any $\delta\in [0,\frac{1}{4})$, the total prior \prb\ density on the incomplete nodes is bounded by $\sum_{\ell\ \text{inc.}} \pi(A_\ell)\leq 2\delta$.
\end{lemma}
\begin{proof}
Let $\ell$ be an incomplete node with state $(A_\ell, p_\ell)$ and write $p=p_\ell$ for simplicity.
Then, the error \prb\ contributed by $\ell$ is
\begin{align*}
\sum_{h\in A_\ell} \pi(h) \cdot (1-p(h)) &= \sum_{h\in A_\ell} \pi(h) - \sum_{h\in A_\ell} \pi(h)\cdot p(h) \\
&= \pi(A_\ell) - \frac{1}{n}\sum_{h\in A_\ell} p(h) \\
&= \frac{|A_\ell|}{n} - \frac{1}{n} \geq \frac{1}{2} \pi(A_\ell),
\end{align*}
where the last \ineq\ follows since  $|A_\ell|\geq 2$.
By the definition of $\delta$-PAC-error, \ift\ 
\[\delta \geq \sum_{\ell\ \text{inc.}} \sum_{h\in A_\ell} \pi(h) \cdot (1-p(h))\geq \frac{1}{2}\sum_{\ell\ \text{inc.}} \pi(A_\ell),\]
i.e., $\sum_{\ell\ \text{inc.}} \pi(A_\ell)\leq 2\delta$. 
\end{proof}

{\bf Proof of Proposition~\ref{lem:chance_constr}.}
It suffices to show how to convert a \dec\ tree $\ho{T}$ with $\delta$-total-error to one with $0$-total-error, without increasing the cost by too much.
Consider each incomplete node $A_\ell$ in $\ho{T}$.
We will replace $A_\ell$ with a (small) \dec\ tree that uniquely identifies a \hypo\ in $A_\ell$.
Consider any distinct \hypos\ $g,h\in A_\ell$. Then by Assumption~\ref{assu1}, there is an action $a\in A$ with $d(g,h;a)\geq s$. 
So if we select $T_a$, then by Hoeffding bound (Theorem~\ref{thm:hoeffding}), we have that with high probability at least one of $g$ and $h$ will be eliminated, and the number of alive \hypos\ in $A_\ell$ reduces by  at least $1$.
Thus, by repeating this procedure iteratively for at most $|A_\ell|-1$ times, we can identify a unique \hypo.
Since each test $T_a$ corresponds to selecting $a$ for $c(a)=s(a)^{-1}\log (|H|/\delta) \leq s^{-1}\log (|H|/\delta)$ times in a row, this procedure increases the total cost by $\sum_{\ell\ \text{inc.}} \pi(A_\ell) \cdot (|A_\ell| \cdot s^{-1}\log (|H|/\delta)$.
Therefore,
\begin{align}\label{eqn:feb11}
ODT^*_0 &\leq ODT_\delta^* + \sum_{\ell \text{ inc.}}\pi(A_\ell) |A_\ell| s^{-1}\log \frac{|H|}{\delta}\notag\\
&= ODT_\delta^* + \sum_{\ell \text{ inc.}} \pi(A_\ell)  |H|\pi(A_\ell) \cdot s^{-1}\log\frac{|H|}{\delta}\notag\\
&= ODT_\delta^* + O\big(s^{-1} |H|\log \frac{|H|}{\delta} \cdot \sum_{\ell \text{ inc.}}\pi(A_\ell)^2 \big).
\end{align}
Since $\sum_{\ell\ \text{inc.}} \pi(A_\ell)\leq 2\delta$ and each $\pi(A_\ell)$'s is non-negative, we have $\sum_{\ell\ \text{inc.}} \pi(A_\ell)^2\leq \big(\sum_{\ell\ \text{inc.}} \pi(A_\ell)\big)^2 \leq 4\delta^2$.
Further, by Pinsker's \ineq, we have $ODT_\delta^*=\Omg(s^{-1}\log \frac{|H|}{\delta}).$
Combining these two facts with Equation (\ref{eqn:feb11}), we obtain 
$ODT^*_0 \leq \big(1+O(|H|
\delta^2)\big)\cdot ODT_\delta^*. $ 
\qed

The following lemma can be proved using the same idea of the proof of Proposition~\ref{prop:convert}. \begin{lemma}\label{lem:odt_delta_total}
$ODT^*_\delta\leq O\big(s^{-1}\log (|H|/\delta) \big) OPT^{FA}_\delta$.
\end{lemma}


Now we are ready to show Theorem~\ref{thm:odt}. 
\begin{align*}
GRE 
&\leq O(\log |H|)\cdot ODT^*_0 \quad & (\text{Theorem~\ref{thm:modt}})\\
&\leq O\big((1+O(|H|
\delta^2)\big)\log |H|)\cdot ODT_\delta^*
\quad & (\text{Lemma~\ref{lem:chance_constr}}) \\&
\leq  O\big((1+O(|H| 
\delta^2)\big)s^{-1}\log ^2 \frac{|H|}{\delta}\log |H|)\cdot OPT^{FA}_\delta. \quad & (\text{Lemma~\ref{lem:odt_delta_total}})
\end{align*}
The above proof can be adapted to the partially \adap\ version \strfwd ly as follows.
Observing that partially adaptive algorithms can be viewed as a special case of the fully adaptive, 
we can define $ODT_{0,\PA}^*$ and $ODT_{\delta,\PA}^*$  (analogous to $ODT_0^*$ and $ODT_\delta^*$) for the partially adaptive version, as the optimal cost of any partially \adap\ \dec\ tree with 0 or $\delta$ error.
By replacing $ODT_\delta^*$ and $ODT_0^*$ with $ODT_{0,\PA}^*$ and $ODT_{\delta,\PA}^*$, 
one may immediatly verify that \ineqs\ in Lemma~\ref{lem:chance_constr} and \ref{lem:odt_delta_total} hold for the partially \adap\ version.
Furthermore, the first inequality above can be established for the partially adaptive version by replacing Theorem \ref{thm:modt} with Theorem~\ref{thm:inz}, hence completing the proof.
\section{Partially Adaptive Algorithm in Experiments}
\label{appdx:partially adaptive algorithm}
In our synthetic experiments, we implement Algorithm~\ref{alg:rnb} described in Section~\ref{sec:partial} exactly, and set the boosting factor, $\alpha$, to be 1. In our real-world experiments, we consider a variant of our algorithm where the boosting intensity is now built-in in the algorithm, and breaking ties according to some heuristic. Algorithm~\ref{alg:rnb exp} describes our modified algorithm. In particular, we consider the amount of boosting as a built-in feature of the algorithm. We first generate a sequence of actions of length $\eta$ for some large $\eta$ value (with replacement) and then truncate the sequence to the minimum length to include all unique actions that have appeared in the sequence. When all actions in sequence $\sigma$ has performed and we did not reach the target accuracy, then we repeat the entire sequence again. Our partially adaptive algorithm on the COSMIC data was generated by initializing $\eta$ to be 800. Across all accuracy levels, the maximum truncated sequence length is 97. 

\begin{algorithm}[t]
\caption{\bf{Partially Adaptive Algorithm in the COSMIC Experiment}}
\begin{algorithmic}[1]
\label{alg:rnb exp}
\STATE{\textbf{Parameters}: Coverage saturation level $B>0$ and maximum sequence length $\eta > 0$.}
\STATE{\textbf{Input}: ASHT instance $(H,A,\pi,\mu)$ }
\STATE{\textbf{Output}: actions sequence $\sigma$}
\STATE{\textbf{Initialize}: $\sigma\lar\emptyset$} \qquad\qquad\quad\quad\quad\quad\quad\quad\quad\quad\quad\quad\quad\quad\quad\quad \% Store the selected of actions
\STATE{\textbf{for}}  {$t=1,2,..., \eta$} \textbf{do} \quad\quad\quad\quad\quad\quad\quad\quad\% {\bf Rank:} Compute a \xulie\ of actions of length $\eta$
\STATE{\quad $S\lar \{\sigma(1),...,\sigma(t-1)\}$.} \quad\quad\quad\quad\quad\quad\quad\quad\quad\quad\quad\quad\quad\quad\quad\quad \%{ Actions selected so far}
\STATE{\quad {\bf for} $a\in A$ \textbf{do},  
\quad\quad\quad\quad\quad\quad\quad\quad\quad\quad\quad\quad\quad\quad\quad\quad\quad\quad \% Compute scores for each action
\vspace{-5pt}
\[\mathrm{Score}(a;S) \lar \sum_{h: f^B_h(S)<1} \pi(h) \frac{f^B_h(S\cup \{a\})- f^B_h(S)}{1-f^B_h(S)}.\]
\vspace{-5pt}
}
\STATE{\quad {\bf end for}}
\STATE{
\quad $\sigma(t) \lar \argmax\{ \mathrm{Score}(a;S): a\in  A\}.$
} \qquad \% Select the greediest action and break ties according the heuristic described in Algorithm~\ref{alg:adaptive_exp}
\STATE{\textbf{end for}}
\STATE{Let i be the largest index for which the an action appears the first time in sequence $\sigma$, then we return the sequence $(\sigma(1), ...., \sigma(i))$.}
\end{algorithmic}
\vspace{-3pt}
\end{algorithm}
\vspace{-5pt}
\section{Fully Adaptive Algorithm in Experiments}
\label{appdx:fully adaptive algorithm}
Similar to NJ's algorithm, we maintain a probability distribution, $\rho$, over the set of hypotheses to indicate the likelihood of each hypothesis being the true hypothesis $h^*$. A hypothesis is considered to be ruled out at each step if the probability of that hypothesis is below a threshold in $\rho$. Throughout our experiments, we set this threshold to be $\delta/|H|$. At each step, after an action is chosen with certain repetitions and observation(s) is (are) revealed, we update $\rho$ according to the realizations that we observed. Thus, under this setup, a hypothesis that was considered to be ruled out in the previous steps (due to ``bad luck'') could potentially become alive again.

At each iteration, for each action $a\in A$ and $k\in {\bf N}$, we define $T_{a,k}$
to be the meta-test that repeats action $a$ for $k$ times consecutively, and
we define its cost to be $c(T_{a,k}) = k c_a$. By Chernoff bound, with $k$ i.i.d. samples, we may construct a \ci\ of width $\sim k^{-1/2}$. 
This motivates us to rule out the following \hypos\ when $T_{a,k}$ is performed.
Let $\bar\mu$ be the observed mean outcome of these $k$ samples. We define the \elimn\ set to be
\[E_{\bar\mu}(T_{a,k}) :=\{h: |\mu(h,a) - \bar \mu| \geq Ck^{-1/2}\},\]
where C is set to be $1/2$ throughout our experiments.
To define greedy, we need to formalize the notion of bang-per-buck.
\Sps\ $H_{alive}$ is the current set of alive \hypos.
We define the score of a test as the number of alive \hypos\ ruled out in the worst-case over all possible mean outcomes $\bar \mu$. 
Formally, the score of $T_{a,k}$ w.r.t mean outcome $\bar\mu$ is \[\mathrm{Score}_{\bar\mu}(T_{a,k}) = \mathrm{Score}_{\bar\mu}(T_{a,k}; H_{alive}) = \frac{|E(T_{a,k}; \bar \mu) \cap H_{alive}|}{c(T_{a,k})},\]
and define its worst-case score to be
\[\mathrm{Score}(T_{a,k}) = \min \{\mathrm{Score}_{\bar\mu}(T_{a,k}): \bar\mu\in\{0, 1/k, ..., 1\}\}.\]
Our greedy policy simply selects the test $T$ with the highest score, formally, select
\[T_{a,k} = \arg\max \{\mathrm{Score}(T): k\leq k_{\max}, a\in A\}.\]

In the synthetic experiments, we set $k_{\max}=5$. In the real-world experiments, we consider the cases where $k\in\{15, 20, 25, 30\}$ (with $k_{\max} = 30$).

If several actions have the same greedy score, then we choose the action $a^*$ whose sum of the KL divergence of pairs of $\mu(h, a^*)$ is the largest, and breaking ties arbitrarily.

If no action can further distinguish any hypotheses in the alive set, then we set the boosting factor to be 1 and use the above heuristic to choose the action to perform. The algorithm is formally stated in Algorithm~\ref{alg:adaptive_exp}.

\begin{algorithm}
\caption{\bf{Adaptive experiments: FA($k_{\max},\delta$)}}
\begin{algorithmic}[1]
\label{alg:adaptive_exp}
\STATE{\textbf{Parameters}: maximum boosting factor $k_{\max}>0$ and convergence threshold $\delta>0$
}
\STATE{ \textbf{Input}: ASHT instance $(H,A,\pi,\mu)$, current posterior about the true hypothesis vector $\rho$ 
}
\STATE{
\textbf{Output}: the test $T_{a, k}$ to perform in the next iteration
}
\STATE Let $H_{\mathrm{alive}}$ be the set of hypotheses $i$ whose posterior probability $\rho_i$ is above $\delta/|H|$.
\FOR{$k=1,2,..., k_{\max}$} 
\STATE For each $a\in \widetilde A$  define:
\[\mathrm{Score}_{\bar\mu}(T_{a,k}) = \mathrm{Score}_{\bar\mu}(T_{a,k}; H_\mathrm{alive}) = \frac{|E(T_{a,k}; \bar \mu) \cap H_\mathrm{alive}|}{c(T_{a,k})},\]
where $E_{\bar\mu}(T_{a,k}) :=\{h: |\mu(h,a) - \bar \mu| \geq Ck^{-1/2}\}$, and $c(T_{a,k}) = k c_a$. We define the worst-case score of a test to be: \[\mathrm{Score}(T_{a,k}) = \min \{\mathrm{Score}_{\bar\mu}(T_{a,k}): \bar\mu\in\{0, 1/k, ..., 1\}\}.\]
\ENDFOR
\STATE Compute greediest action 
\[G = \arg\max \{\mathrm{Score}(T): k\leq k_{max}, a\in A\}.\]
\IF{the Score of each test in $G$ equals to 0, i.e, no test can further distinguish between the alive hypotheses under $k_{\max}$}
\STATE{we choose the action $a^*$ such that $a^* = \arg\max \sum_{h, g\in H_\mathrm{alive}} \mathrm{KL}(\mu(h, a), \mu(g, a))$, breaking ties randomly, and return $k=1$.}
\ELSE
\STATE if $G$ is a singleton, then we return $G$. Else, we choose the action $a^*$ such that $a^* = \arg\max_{G} \sum_{h, g\in H_\mathrm{alive}} \mathrm{KL}(\mu(h, a), \mu(g, a))$, and breaking ties randomly.
\ENDIF
\end{algorithmic}
\end{algorithm}

\end{document}